\documentclass[journal]{IEEEtran}
\usepackage{cite}
\usepackage{amsmath,amssymb,amsfonts}
\usepackage{algorithmic}
\usepackage{graphicx}
\usepackage{textcomp}
\usepackage{comment}
\usepackage{xcolor}
\usepackage{array}
\usepackage{amsmath,amssymb,amsfonts}
\usepackage{amsthm}
\usepackage{mathrsfs}
\usepackage{CJK}
\usepackage[ruled,linesnumbered]{algorithm2e}
\usepackage{amsmath}
\usepackage{url}
\usepackage{color}
\usepackage{pifont}

\newtheorem{assumption}{Assumption}
\newtheorem{lemma}{Lemma}
\newtheorem{theorem}{Theorem}

\newcommand{\whiteding}[1]{\ding{\numexpr171+#1\relax}}
\newcommand{\tabincell}[2]{\begin{tabular}{@{}#1@{}}#2\end{tabular}}

\allowdisplaybreaks[4]

\ifCLASSINFOpdf
\else
\fi
\begin{document}
\title{Slashing Communication Traffic in Federated Learning by Transmitting Clustered Model Updates}

\author{\IEEEauthorblockN{Laizhong Cui, Senior Member, IEEE,
		Xiaoxin Su,
		Yipeng Zhou\IEEEauthorrefmark{1}, Member, IEEE, 
		Yi Pan, Senior Member, IEEE}
	
	\IEEEauthorblockA{
		\thanks{
		\newline L. Cui, X. Su, are with the College of Computer Science and Software Engineering, Shenzhen University, Shenzhen, PR.China (email: cuilz@szu.edu.cn; suxiaoxin2016@163.com).
		\newline Y. Zhou is with the Department of Computing, FSE, Macquarie University, Australia, 2122 (email: yipeng.zhou@mq.edu.au).
		\newline Y. Pan is with the Georgia State University, USA (email: pan@cs.gsu.edu).
		\newline Y. Zhou is the corresponding author.
		}
	}
} 
\markboth{}{Shell \MakeLowercase{\textit{et al.}}: Bare Demo of IEEEtran.cls for IEEE Journals}
\maketitle

\begin{abstract}
Federated Learning (FL) is an emerging decentralized learning framework through which multiple clients can collaboratively train a learning model. However, a major obstacle that impedes the wide deployment of FL lies in  massive communication traffic. To train  high dimensional machine learning models (such as CNN models), heavy communication traffic can be incurred by exchanging model updates  via the Internet between clients and the parameter server (PS), implying that the network resource can be easily exhausted. Compressing model updates is an effective way to reduce the traffic amount. However, a flexible unbiased compression algorithm applicable for both uplink and downlink compression in FL is still absent from existing works. 
In this work,  we devise the Model Update Compression by Soft Clustering (MUCSC) algorithm to compress model updates transmitted between clients and the PS. In MUCSC, it is only necessary to transmit cluster centroids and the cluster ID of each model update. Moreover, we prove that: 1) The compressed model updates are unbiased estimation of their original values so that  the convergence rate by transmitting compressed model updates is unchanged; 2) MUCSC can guarantee that  the influence of the compression error  on the model accuracy is minimized.  
Then, we further propose the boosted MUCSC (B-MUCSC) algorithm, a biased compression algorithm that can achieve an extremely high compression rate by grouping insignificant model updates into a super cluster. B-MUCSC  is suitable for scenarios with very scarce network resource. Ultimately, we conduct extensive experiments with the CIFAR-10 and FEMNIST datasets to demonstrate that our algorithms can not only substantially reduce the volume of communication traffic in FL,  but also improve the training efficiency in practical networks. 
\end{abstract}


\IEEEpeerreviewmaketitle

\section{Introduction}
With the rapid development of the Internet of Things (IoT) and edge computing \cite{rydning2018digitization}, terminal devices generate a large amount of data through interacting with the environment which can help us train more advanced machine learning models.
However, it is unrealistic to  collect all these data to centrally train machine learning models due to  the limitation of  network capacity and the concern of privacy leakage. 

To train models without intruding user privacy,  Federated Learning (FL) is firstly devised by Google \cite{mcmahan2017federated} and then thrives quickly. Through FL, decentralized clients only interchange model updates with the parameter server (PS) via the Internet to complete a training task so that the original data samples can be maintained locally \cite{nishio2019client}. Due to this fantastic feature, FL is specially applicable for IoT and edge devices to collaboratively train machine learning models \cite{yang2019federated, yu2018federated, samarakoon2018federated}. 

However, a major obstacle that impedes the deployment of FL in real world lies in the massive communication traffic transmitted between clients and the PS. 
In a typical FL system, there may reside thousands of clients \cite{bonawitz2019towards} who need to exchange model updates with the PS for multiple rounds of aggregations (a.k.a. global iterations). However, the network speed, especially the uplink speed, is much slower than the CPU speed \cite{li2014communication}. Meanwhile, the model dimension could be in the magnitude of millions for complicated neural networks \cite{springenberg2014striving}. Thus, it implies that network resource can be easily exhausted by training high dimensional models in FL.  For instance, each client can generate 4MB uplink traffic per global iteration if the model dimension is one million and the size of each model update is 4 bytes. If there are 1,000 clients to communicate with the PS for 100 rounds, the total uplink traffic will be 400GB. The huge traffic volume can saturate the network capacity and choke the PS, and the training time can be prolonged substantially. For many real-time applications, such an inefficient  training process is unacceptable \cite{ananthanarayanan2017real}. 

A number of previous works have been dedicated to  compressing model updates \cite{konevcny2016federated} in FL. These works can be summarized as two kinds. The first kind devised unbiased  algorithms which approximate model updates with unbiased estimation \cite{alistarh2016qsgd}.
In contrast, the second kind designed biased compressing algorithms which are more flexible and can achieve much higher compression rates \cite{sattler2019robust, li2020ggs}. 
{ 
However, a common challenge confronted by the above compression algorithms is that the distribution of model updates is different in every global iteration. It implies that compression algorithms should be adjusted effectively in accordance with the distribution of model updates in each global iteration, which has not been well solved by existing works. 
It is also worth to mention compression algorithms designed for distributed machine learning systems \cite{alistarh2016qsgd, konevcny2018randomized}.} Although, it is possible to apply these algorithms in FL, there is no guarantee for the compression performance and model accuracy in that the sample distribution in FL is non-iid different from that in  distributed learning systems \cite{konevcny201610federated}. In addition, the convergence property of iterative algorithms in FL such as FedAvg \cite{mcmahan2017communication} is different from that in distributed learning systems \cite{li2014scaling}.

{In this work, we propose a novel compression algorithm to reduce the communication volume in FL without compromising the model accuracy, which is named as  Model Update Compression by Soft Clustering (MUCSC).}
{ MUCSC can compress both uplink and downlink model updates. It can achieve a very high compression rate, and meanwhile  guarantee the convergence of trained machine learning models. }  
Specifically,  in a round of global iteration,  each  participating client groups model updates to be transmitted into a small number of clusters. Cluster centroids are determined  adaptively through minimizing a specially  designed compression error function. Then, each client only needs to upload the centroid value of each cluster and the cluster ID of each model update to the PS.
In a similar way, the PS compresses aggregated model updates by grouping all aggregated model updates  into a small number of clusters before it distributes the aggregated model updates to all clients. 

We theoretically prove that the compressed model updates are unbiased estimation of their original values. By leveraging the FedAvg algorithm \cite{mcmahan2017communication}, we further prove that the MUCSC algorithm can achieve the same convergence rate as that without compression. The specially designed compression objective function can minimize the influence of the compression error on the  model accuracy. 
To further  boost  MUCSC, we propose the boosted MUCSC (B-MUCSC) algorithm that groups insignificant model updates with values close to 0 as a super cluster. With this boost, the compression rate can be further improved by more than 10 times.

In a word, the contribution of our work is summarized as below.
\begin{itemize}
    \item In order to slash the communication traffic volume in FL, we propose a novel model update compression algorithm, \emph{i.e.}, MUCSC, {which can compress both uplink and downlink model updates.}
    \item We theoretically prove that MUCSC can achieve the same convergence rate as that without compression. The objective function of MUCSC can minimize the influence of the compression error on the model accuracy. 
    \item A biased algorithm B-MUCSC is further proposed that can boost  the compression rate by more than 10 times. 
    \item At last, experiments are conducted with the {FEMNIST} and CIFAR-10 datasets to demonstrate the superiority of MUCSC and B-MUCSC. Experiment results show that our algorithms can reduce the volume of the communication traffic in FL by more than 80\% and 95\% respectively.  
\end{itemize}

The rest of the paper is organized as follows. Related works are discussed in Sec.~\ref{RelatedWork}. Sec.~\ref{Preliminary} introduces the preliminary knowledge before MUCSC is presented in Sec.~\ref{MUCSCAlgorithm}. The convergence rate using MUCSC for compression  is analyzed in Sec.~\ref{ConvergenceOfMUCSC} and its performance is evaluated in Sec.~\ref{Performance}. In the end, we conclude our work  in Sec~\ref{Conclusion}.

\section{Related Work} \label{RelatedWork}
\subsection{Overview of Federated Learning}
Federated learning can train the model by allowing the clients to transmit only the model updates, thereby protecting the privacy of the data \cite{mcmahan2017communication}. The work \cite{yang2019federated} provided an introduction to the definition, architecture, and applications of FL.
Considering the complex of the mobile edge network, the challenges of FL and corresponding strategies were discussed regarding  communication costs, resource allocation and privacy in \cite{lim2020federated}.

Currently, researchers have   develop different applications for data-sensitive scenarios in FL, such as providing predictive models for health diagnosis \cite{brisimi2018federated} and promoting cooperation among multiple government agencies \cite{verma2018federated}. Google proposed the FedAvg framework in 2017 \cite{mcmahan2017communication} which was applied to Gboard to improve the word prediction model \cite{hard2018federated}.

The work \cite{li2019convergence} derived the convergence rate of the FedAvg algorithm and also discussed the influence of different parameters on the model convergence. 
Wang et al. \cite{wang2019adaptive} discussed how to perform adaptive FL in resource-constrained edge computing systems.
An algorithm called FEDL was proposed in \cite {tran2019federated}, which discussed the optimization of the resource allocation of FL in wireless networks.
The work \cite{chen2019performance} formulated the resource allocation and user selection problems of FL in the wireless network 
and \cite{cai2020dynamic} proposed a dynamic sample selection optimization algorithm called FedSS to solve the heterogeneous data in FL.

\subsection{Reducing Communication Cost in FL}
Considering the huge traffic transmission in the training process of FL, there has been a lot of research work dedicated to reducing the communication cost in FL, in order to accelerate the training speed.

The authors of \cite {yao2018two} proposed a two-stream model with a maximum average difference constraint which forced the local two-stream model to learn more knowledge from other devices to reduce the number of communications.
The authors of \cite{liu2020client} proposed a client-edge-cloud hierarchical FL system which can shorten the propagation delay from clients to the PS, and thus alleviated the communication load of the PS.
The FedCS protocol was designed in \cite{nishio2019client}, which can select clients according to the resource status of different clients, so that the PS can aggregate the model updates of as many clients as possible.

Many researchers designed different compression algorithms to reduce the traffic transmitted to the PS and speed up the training process. Basically, there are two kinds of compressing algorithms: unbiased and biased.

Unbiased algorithms are briefly discussed as below. 
The work \cite{wen2017terngrad} designed the TernGrad algorithm. It quantified all gradients to 3 values, which greatly reduced the amount of data that need to be uploaded. 
Wangni et al. \cite{wangni2018gradient} proposed a sparse algorithm, which randomly discarded the coordinates of the stochastic gradient vector and appropriately amplified the remaining coordinates to ensure that the sparse gradient was unbiased and reduce communication costs.
The authors in \cite{alistarh2017qsgd} proposed an algorithm to Quantized SGD (QSGD), which allowed users to adjust the number of bits sent in each iteration based on the tradeoff between the network bandwidth and the convergence time.
The work \cite{suresh2017distributed} developed probabilistic quantization algorithms and explained the mean square error introduced by these algorithms. 
ATOMO was a general framework for sparsification of stochastic gradients, which was introduced by \cite{wang2018atomo}. ATOMO established and optimally solved a meta-optimization by minimizing the variance of the sparsified gradient. 

 {
Although unbiased compression algorithms can guarantee the convergence of FL training, the compression error can result in much lowered model accuracy. 
It is still unknown which unbiased algorithm  achieves the best model accuracy.
}

Compared with unbiased compression algorithms, biased compression algorithms  {cannot guarantee the convergence of FL training}, but can achieve much higher compression rates.
The stochastic gradient descent (SGD) combined with k-sparsification  for compression has been analyzed in \cite{stich2018sparsified}.
The authors of \cite{han2020adaptive} designed an adaptive gradient sparse algorithm that can adaptively select gradients using an online learning algorithm. 
In \cite{luping2019cmfl}, the CMFL algorithm was designed, which removed updates that were not related to the global model from uploading, and thereby reduced the communication overhead. 
The work \cite{alistarh2018convergence} proved that for distributed machine learning, Top-K gradient sparsification combined with local error correction can provide proof of convergence for smooth functions. 
The authors of \cite{shi2019layer} proposed a novel hierarchical adaptive gradient sparse algorithm called LAGS-SGD, and proved the convergence of the algorithm under weak analysis assumptions.


 {It worth to mention that  algorithms designed for parameter compression in distributed machine learning cannot be applied for FL straightly because the  data distribution in FL is non-iid, which should be taken into account by compression algorithms in FL.
In this paper, we propose a novel unbiased compression algorithm for FL by soft clustering model updates. Cluster centroids can be generated adaptively in each global iteration according to the distribution of model updates so as to minimize the influence of the  compression error on the final model accuracy. 
}

\section{Preliminary} \label{Preliminary}
In FL, data samples are  distributed among multiple clients. We assume that the total number of clients is $N$ and the data set on client $i$ is denoted by $\mathcal{S}_i$.
In order to make these clients cooperate to train a machine learning model, clients will download the latest model updates from the PS and use their local data to update  the model by minimizing a  predefined loss function. Without loss of generality, the local loss function is defined using a local data sample batch $\mathcal{B}$ of  size $B$. That is
\begin{eqnarray}
    F_i(\mathbf{w}, \mathcal{B}_i) = \frac{1}{B} \sum_{\forall \xi_j\in\mathcal{B}_i} f(\mathbf{w}, \xi_j).
    \label{EQ:LocalLoss}
\end{eqnarray}
Here, $\mathcal{B}_i$ represents a small sample batch selected from client $i$'s local dataset $\mathcal{S}_i$. $f(\mathbf{w}, \xi_j)$ represents the loss function calculated by a specific sample $\xi_j$ and the model parameters $\mathbf{w}$.
The goal of FL is to train the model parameters that minimize the global loss function, which is defined as
\begin{eqnarray}
    \min_{\mathbf{w}}\quad F(\mathbf{w}) = \sum_{i=1}^N p_i F_i(\mathbf{w},\mathcal{S}_i),
    \label{EQ:GlobalLoss}
\end{eqnarray}
where $p_i$ represents the weight of client $i$, which is usually set as the ratio of the dataset size of client $i$ to the total dataset size, \emph{i.e.}, $p_i = \frac{|\mathcal{S}_i|}{\sum_{i'=1}^N|\mathcal{S}_{i'}|}$.

In this paper, we use the FedAvg \cite{mcmahan2017communication} algorithm as the basic framework of FL. To smooth the subsequent explanation, we first introduce the FedAvg algorithm.

In FedAvg, the PS  randomly selects a number of clients to participate in training in each round of communication. The selected clients will use local data to perform $E$-round local iterations. After the local training, each client  uploads the model updates to the PS for aggregation. Then, the PS distributes aggregated model updates to clients to kick off a new round of global iteration. 

Let $\mathbf{w}_t$ denote the model parameters in the $t^{th}$ iteration and $\mathbf{w}^i_t$ denote the model parameters on the client $i$ at iteration $t$. The training process on client $i$ is conducted as follows
\begin{eqnarray}
    \label{EQ:LocalTrain}
    \mathbf{w}^i_{t+1} = \mathbf{w}^i_{t}-\eta_{t} \nabla F_i(\mathbf{w}^i_{t},\mathcal{B}^i_{t}),
\end{eqnarray}
where $\eta_t$ represents the learning rate at iteration $t$ and $\mathcal{B}^i_t$ represents the mini batch of data samples randomly selected by client $i$ at iteration $t$.
After executing $E$ local iterations, the selected client uploads the model updates to the PS for aggregation. The aggregation is operated  on the server side as
\begin{eqnarray}
    \label{EQ:GlobalAggregation}
    \mathbf{w}_{t+E} =\mathbf{w}_{t} \!\!-\!\! \sum_{i\in \mathcal{K}_{t}}\frac{|\mathcal{S}_i|}{\sum_{i^{'}\in\mathcal{K}_{t}}|\mathcal{S}_{i^{'}}|}(\sum_{j=t}^{t+E-1}\eta_{j} \nabla F_i(\mathbf{w}^i_{j},\mathcal{B}^i_{j})),
\end{eqnarray}
where $\mathcal{K}_{t}$ represents $K$ clients randomly selected at global iteration $t$.
To facilitate the understanding of our algorithms and analysis, we provide a notation list in Table~\ref{NotationList}.

\begin{table}[h]
    \centering
    \caption{Notation List}
    \begin{tabular}{|m{1.5cm}<{\centering}|m{6cm}<{\centering}|}
         \hline
         Notation & Meaning \\
         \hline
         $\mathcal{B}$ & a batch of data samples with size $B$ \\
         \hline
         $p_i$ & the weight of client $i$ among all clients \\
         \hline
         $\mathcal{S}_i$ & dataset in client $i$ \\
         \hline
         $\mathbf{U}^i_{t+1}$($\mathbf{D}_{t+1}$) & model updates uploaded(downloaded) by client $i$(server) in the $t+1^{th}$ global iteration\\
         \hline
         d & dimensions of model updates \\
         \hline
         $Z_U$($Z_D$) & number of centroids at upload(download)\\
         \hline
         $r_z$ & the value of the $z^{th}$ centroid \\
         \hline
         $\widetilde{\mathbf{U}}^i_{t+1}$($\widetilde{\mathbf{D}}_{t+1}$) & compressed model updates randomly generated based on input $\mathbf{U}^i_{t+1}$($\mathbf{D}_{t+1}$)'s\\
         \hline
         $\mathcal{L}_z$ & set of model updates between $r_z$ and $r_{z+1}$\\
         \hline
         $\mathcal{I}$ & set of rounds for global aggregation \\
         \hline
         $\mathcal{K}_{t+1}$ & set of clients selected to participate the $(t+1)^{th}$ global iteration\\
         \hline
         K & number of clients participating each global iteration\\
         \hline
         $\Gamma$ & quantified heterogeneity of the non-iid data distribution\\
         \hline
         $\Lambda$ & compression rate\\
         \hline
    \end{tabular}
    \label{NotationList}
\end{table}

\section{MUCSC Algorithm Design} \label{MUCSCAlgorithm}
In this section, we introduce the MUCSC algorithm which can compress the model updates transmitted between the PS and clients. 
We mainly introduce the compression of  model updates to be uploaded from clients to the PS. The compression of model updates distributed from the PS to clients can be operated in a similar manner.

\subsection{Uplink Compression}
We use $\|.\|$ to represent the $\ell_2$ norm. We define the  model updates obtained after the client $i$ conducts $E$ local iterations starting from iteration $t$ as 
\begin{eqnarray}
    \label{EQ:ModelUpdates}
    \mathbf{U}^i_{t+E} = \mathbf{w}^i_{t} - \mathbf{w}^i_{t+E}= \sum_{j=t}^{t+E-1} \eta_{j} \nabla F_i(\mathbf{w}^i_{j},\mathcal{B}^i_{j}).
\end{eqnarray}
The dimension of the trained model is denoted by $d$, and hence the dimension of model updates is also $d$. Without the compression operation, each client needs to upload $hd$ bytes of data (assuming that each floating point number is represented by $h$ bytes) to the PS in each global iteration. Suppose that there exist $N$ clients. It implies that the volume of the communication traffic in FL is with the magnitude of $d\times N$, and thus training high dimensional learning models can be very network resource consuming.


In order to reduce the traffic volume so as to speed up the FL training process, we can classify all model updates on a particular client into $Z$ clusters. Here $Z$ is a number much smaller than $d$. Rather than uploading $hd$ bytes, the client can upload $Z$ centroid values and the cluster ID of each model update. In this way, the volume of the upload traffic is reduced to  $hZ$ bytes representing $Z$ centroid values plus $\frac{\lceil \log_2Z \rceil d}{8}$ bytes representing cluster IDs.

Apparently, how to choose $Z$ centroid values is very essential, which can heavily determine the compression error and the convergence performance. There are two facts we need to consider. Firstly, to make the FedAvg algorithm converge, we need to ensure that the compressed model updates are unbiased estimation of their original
values.  Secondly, given a fixed $Z$, we need to determine $Z$ centroid values so that the influence of the compression error on the  converge is minimized. 

Without the loss of generality, we suppose that the model updates to be compressed are denoted by the vector $\mathbf{U}$ of dimension $d$. Let $U_m$ denote the $m^{th}$ element in $\mathbf{U}$. We assume that $Z$ centroids are sorted in ascending order with values $r_1 < r_2 < \dots < r_Z$ and $r_1$ and $r_Z$ are the minimum and maximum values in the vector $\mathbf{U}$, respectively. 
We also defined a function $\phi(U_m)=z$ which means $r_z\le U_m$ and $r_{z+1}\ge U_m$. When compressing a vector, we use this function to lookup the centroid ID for a model update $U_m$ in the vector. Then we perform the compression operation with $U_m$ as
\begin{equation}
    \widetilde{U}_{m}=\left\{
        \begin{aligned}
        & r_z, \quad \textit{with prob.}\ \frac{r_{z+1}-U_{m}}{r_{z+1}-r_z},\\
        &r_{z+1}, \quad \textit{with prob.}\ \frac{U_{m}-r_z}{r_{z+1}-r_z}. \\
        \end{aligned}
        \right.
    \label{EQ:SoftClusterQuantization}
\end{equation}
Here $\mathbb{E}$ indicates the expectation of a random variable. It is easy to verify that $\mathbb{E}\widetilde{U}_{m}=U_m$, which implies that the compressed value $\widetilde{U}_m$ is the unbiased estimation of $U_m$.
According to \eqref{EQ:SoftClusterQuantization}, we find that the variance between the compressed $\widetilde{U}_{m}$ and the original $U_m$ is
\begin{eqnarray}
    \label{EQ:CrossDfortl}
    J_{m} = \mathbb{E}\left(\widetilde{U}_{m} -U_{m}\right)^2= \left(r_{z+1}-U_{m}\right)\left(U_{m}-r_z\right),
\end{eqnarray}
where $z=\phi(U_m)$. In fact, to minimize the influence of the compression error on the convergence, we should minimize the variance between $\widetilde{U}_{m}$ and $U_m$. The detailed proof will be presented in the next section. By considering all $d$ model updates, the loss function we should minimize in the MUCSC algorithm is defined as
\begin{eqnarray}
    \label{EQ:VarianceOfCompression}
    J = \mathbb{E}\|\widetilde{\mathbf{U}} -\mathbf{U}\|^2  = \sum_{m=1}^d J_{m}.
\end{eqnarray}

\begin{algorithm}[t]
    \label{MUCSC}
    \SetKwBlock{E}{E-step}{}
    \SetKwBlock{M}{M-step:Update the parameters}{}
    \caption{The Model Update Compression By Soft Clustering (MUCSC) Algorithm}
    \LinesNumbered
    \KwIn{A vector $\mathbf{U}$; Number of centroids $Z$}
    \KwOut{$r_1,\dots, r_Z$, compressed vector $\widetilde{U}_{m}$'s }
    {Randomly initialize $r_2, \dots, r_{Z-1}$};\\
    \While {\eqref{EQ:VarianceOfCompression}\ does\ not\ converge}{
        \E{
            \For{z = 2,...,Z-1}{
                $r_z=r_z-\alpha \frac{\partial J}{\partial r_{z}}$;
            }
        }
        \M{
            Set $\mathcal{L}_z=\emptyset$ for $z=1,\dots, Z$;\\
            \For{m = 1,2,...,d}{
                Update $\phi(U_{m})$ according to $r_1,\dots, r_Z$;\\
                Update $\mathcal{L}_z$ for $z =1,\dots, Z$; 
            }
        }
    }
    Generate compressed vector $\widetilde{U}_{m}$'s according to \eqref{EQ:SoftClusterQuantization};\\
    \Return{$r_1, \dots, r_Z$ and $\widetilde{U}_{m}$'s;}
\end{algorithm}

To determine the values of the $Z$ centroids so as to minimize the loss function defined in \eqref{EQ:VarianceOfCompression}, we employ the EM algorithm to solve this soft clustering problem. 
 {The EM algorithm first computes the gradients based on the given centroid; then updates centroids and the sets of model updates division based on the gradient computed in the previous step, and then iterates repeatedly until a certain number of iterations. The steady ascending steps of the EM algorithm can find the optimal centroid very reliably.} 
In E step, we use the gradient descent method to update $Z$ centroid values. According to the loss function \eqref{EQ:VarianceOfCompression}, the update formula should be
\begin{equation}
    \begin{split}
        r_{z} &= r_{z} - \alpha \frac{\partial J}{\partial r_{z}} \\
        &= r_{z} - \alpha \!\Big(\!\!\sum_{\forall m\in \mathcal{L}_{z-1}}\!\!\!\!(U_{m}-r_{z-1})-\!\!\!\!\sum_{\forall m\in \mathcal{L}_z}(r_{z+1}-U_{m})\Big).
    \end{split}
    \label{UpdateParam}
\end{equation}
Here $\alpha$ is the learning rate and $\mathcal{L}_{z}$ is the set of elements satisfying $\phi(U_m)=z$. 

The specific operation process of the EM algorithm is presented  in Algo. \ref{MUCSC}.  {It should be noted that $r_1$ and $r_Z$ are the minimum and maximum values of the elements in the vector $\mathbf{U}$, and their values will not be changed by the EM algorithm.  So we first initialize the other $Z-2$ centroids (\emph{i.e.}, line 1 in Algo. \ref{MUCSC}). In E step, we update cluster centroids according to \eqref{UpdateParam} (\emph{i.e.}, lines 3-6). In M step, we need to recalculate $\phi(U_m)$ and $\mathcal{L}_z$ based on the updated centroids (\emph{i.e.}, lines 7-12). 
}

\subsection{Downlink Compression}
Until now, our discussion is only about the compression of model updates to be uploaded by each client. In fact, the MUCSC algorithm can also be applied to compress the aggregated model updates distributed from the PS to clients. Let $\mathbf{D}_t$ denote the aggregated model updates  on the PS at the global iteration $t$. According to \eqref{EQ:GlobalAggregation},  $\mathbf{D}_t$ can be expressed as 
\begin{eqnarray}
    \label{EQ:DefineDt}
    \mathbf{D}_{t} = \sum_{i\in \mathcal{K}_{t}}\frac{|\mathcal{S}_i|}{\sum_{i^{'}\in\mathcal{K}_{t}}|\mathcal{S}_{i^{'}}|} \widetilde{\mathbf{U}}_t,
\end{eqnarray}
The PS compress $\mathbf{D}_{t} $ into $\widetilde{\mathbf{D}_{t}} $, which is then distributed to all clients. Each client obtains $\mathbf{w}_{t}$ through computing
\begin{eqnarray}
    \label{EQ:ComputeWtE}
    \mathbf{w}_{t} = \mathbf{w}_{t-E} -  \widetilde{\mathbf{D}}_t,
\end{eqnarray}

Obviously, the compression from $\mathbf{D}_t$ to $\widetilde{\mathbf{D}}_t$ can be conducted by executing the MUCSC algorithm in a similar way to compress $\mathbf{U}$. Due to the limited space, we omit the repetitive description to compress $\mathbf{D}_t$.
To avoid confusion, we let $Z_U$ and $Z_D$ denote the numbers of centroids for uplink and downlink compression operations, respectively.


 {To have a holistic overview of our algorithm, the running process of MUCSC is presented  in Fig.~\ref{MUCSC_Process}.}
\begin{figure}[h]
	\centering
	\includegraphics[width=\linewidth]{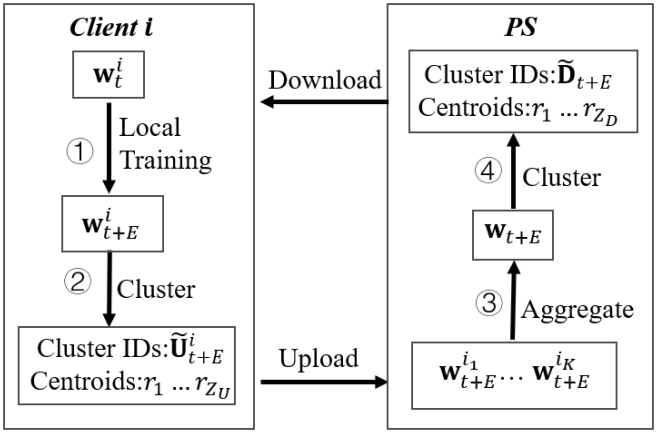}
	\caption{ {MUCSC Process: \whiteding{1} The selected client uses the compressed data downloaded by the PS to obtain the local model, and uses the local data to conduct E-epoch training on the local model. \whiteding{2} The selected clients cluster the model updates, that is, the EM algorithm is used to obtain the centroids with the smallest compression error and the centroids are used to compress the model updates. \whiteding{3} The PS uses the compressed data uploaded by the clients to compute their local model and computes the global model and model updates. \whiteding{4} The PS uses the EM algorithm to compute the centroids that minimize the compression error and compresses the model updates through the centroids. The PS broadcasts the compressed data to all clients and selects some clients to participate in the training.} }
	\label{MUCSC_Process}
\end{figure}

\section{ Convergence Rate Analysis} \label{ConvergenceOfMUCSC}
In this section, we analyze the convergence rate of FedAvg if model updates are compressed by the MUCSC algorithm. Our analysis is conducted under two scenarios: full participation mode and partial participation mode.


\subsection{Assumptions}

Similar to the convergence rate analysis conducted in previous works  \cite{li2019convergence, mishchenko2019distributed, yu2019parallel, dinh2019federated}, we make the following assumptions.
\begin{assumption}
\label{Assump:ConSmoo}
The loss functions, \emph{i.e.}, $F_1, F_2,\dots, F_N$ are all $L$-smooth and $\mu$-strongly convex. In other words, given $\mathbf{v}$ and $\mathbf{w}$, we have $F_i(\mathbf{v}) \le F_i(\mathbf{w}) + (\mathbf{v}-\mathbf{w})^T\nabla F_i(\mathbf{w})+\frac{L}{2}||\mathbf{v}-\mathbf{w}||^2$ and $F_i(\mathbf{v}) \ge F_i(\mathbf{w}) + (\mathbf{v}-\mathbf{w})^T\nabla F_i(\mathbf{w})+\frac{\mu}{2}||\mathbf{v}-\mathbf{w}||^2$.
\end{assumption}

\begin{assumption}
\label{Assump:LocalVar}
Let $\xi^i_t$ denote the sample randomly and uniformly selected from client $i$. The variance of the stochastic gradients in each client is bounded:
$\mathbb{E}[\|\nabla F_i(\mathbf{w}^i_t, \xi^i_t)-\nabla F_i(\mathbf{w}^i_t)\|^2] \le \sigma_i^2$ for $i=1, 2,\dots, N$.
\end{assumption} 

\begin{assumption}
\label{Assump:BoundG}
The expected square norm of stochastic gradients is uniformly bounded, \emph{i.e.},  $\mathbb{E}[\|\nabla F_i(\mathbf{w}^i_t, \xi^i_t)\|^2] \le G^2$ for all $i = 1,2,\dots,N$ and $t = 0, 1,\dots, T-1$
\end{assumption}

As we all know, the data sample distribution among clients in FL is non-iid. Therefore, we need to quantify the degree of the non-iid sample distribution. We define $\Gamma = F^*-\sum_{i=1}^Np_iF^*_i$ to quantify the degree of non-iid distribution of clients' data where $F^*$ and $F^*_i$ are the optimal values of $F$ and $F_i$ respectively. At the same time, we define the set $\mathcal{I}$ to represent the indices  of global iterations in the FL training process, \emph{i.e.}, $\mathcal{I}=\{E, 2E, \dots\}$.

\subsection{Analysis of Full Participation Mode}

We first discuss the full client participation mode, in which the PS always involves all clients to participate in each round of global iteration. 
In this case, if $t+1\in\mathcal{I}$, each client will use the MUCSC algorithm to compress the model updates obtained by her own dataset and then upload the model updates to the PS. After the PS receives all the model updates from all clients, it proceeds to aggregate model updates. Then,  the PS also runs the MUCSC algorithm to compress the aggregated model updates and distributes the compressed aggregated model updates to all clients to kick off a new round of global iteration. 

Recall that in \eqref{EQ:ModelUpdates} we have used $\mathbf{U}^i_{t+1}$ to represent the model updates uploaded by client $i$ at global iteration $t+1$. Let $\mathbf{v}^i_{t+1}$ represent the model parameters obtained by client $i$ after executing a round of local iteration with the  model parameters $\mathbf{w}^i_{t}$. Therefore, the specific model update process is as follows.
\begin{equation}
    \label{EQ:LocalUpdate}
    \begin{split}
        \mathbf{v}^i_{t+1}=\mathbf{w}^i_{t}-\eta_t\nabla F_i(\mathbf{w}^i_t,\mathcal{B}^i_t).
    \end{split}
\end{equation}
Then, $\mathbf{w}_{t+1}^i$ can be computed as:
\begin{equation}
    \mathbf{w}^i_{t+1}=\left\{
        \begin{aligned}
        &\mathbf{v}^i_{t+1}, \quad if\ t+1\ \notin \mathcal{I},\\
        &\mathbf{w}_{t+1-E}^i - \widetilde{\mathbf{D}}_{t+1}, \quad if\ t+1\ \in \mathcal{I}. \\
        \end{aligned}
        \right.
\end{equation}
Recall that $\widetilde{\mathbf{D}}_{t+1}$ denotes the compression of  $\mathbf{D}_{t+1}$, and $\mathbf{D}_{t+1}$ represents the aggregated model updates obtained by the PS. It turns out that $\mathbf{D}_{t+1}=\sum_{i=1}^Np_i\widetilde{\mathbf{U}}^i_{t+1}$, where $\widetilde{\mathbf{U}}^i_{t+1}$ is the compressed model updates uploaded by client $i$.

 {We analyze the scenario that all devices participate in each round of global iteration. As we have described in the last section,  the MUCSC algorithm runs for $\frac{T}{E}$ times (\emph{i.e.}, executed once in each global iteration). Here we assume that $T$ is divisible by $E$.}

\begin{theorem}
\label{Theorem1}
Let Assumptions \ref{Assump:ConSmoo}-\ref{Assump:BoundG}  hold under the full participation mode. For a decreasing learning rate, we define $\eta_t=\frac{\beta}{t+\gamma}$, where $\beta>\frac{1}{\mu}$ and $\gamma>0$ such that $\eta_1\leq\min\{\frac{1}{\mu}, \frac{1}{4L}\}$ and $\eta_t\leq2\eta_{t+E}$. Then the convergence rate of FedAvg by using MUCSC for model update is: 
\begin{equation*}
    \mathbb{E}\left[F(\mathbf{w}_{T})\right]-F^*\leq \frac{L}{2}\frac{v_T}{T+\gamma}, 
\end{equation*}
where $v_T=\max\{\frac{\beta^2\tau}{\beta\mu-1}+\sum_{t=1}^T[(t+\gamma)\psi_t], \gamma\|\mathbf{w}_0-\mathbf{w}^*\|^2\}$, $\tau=\sum_{i=1}^Np_i^2\sigma_i^2+6L\Gamma+8(E-1)^2G^2$ and $\psi_{t+1} = J^D_{t+1} + \sum_{i=1}^Np_i^2J^i_{t+1}$.
\end{theorem}
Please refer to Appendix \ref{ProofOfTheorem1And2} for the detailed proof.

\noindent{\bf Remark:}
It is notable that $J^D_{t+1}$ and $J^i_{t+1}$ are the loss function of the  MUCSC algorithm when it is executed by the PS and the client $i$ respectively. Through Theorem \ref{Theorem1}, we can link the compression errors with  the convergence rate. Obviously, both $J^D_{t+1}$ and $J^i_{t+1}$ can affect the convergence. 
To minimize the influence of the compression errors, we should minimize  $J^D_{t+1}$ and $J^i_{t+1}$, which are just the objective of the MUCSC algorithm. 

Here $J^D_{t+1}$ and $J^i_{t+1}$ are not constant values. To more precisely derive the convergence rate, we need to bound both $J^D_{t+1}$ and $J^i_{t+1}$. We present the main result of the convergence rate under the full participation mode as below.  
\begin{theorem}
\label{Theorem2}
Let Assumptions \ref{Assump:ConSmoo}-\ref{Assump:BoundG} hold under the full participation mode. Let $\kappa=\frac{L}{\mu},\gamma=\max\{8\kappa, E\}$ and $\eta_t=\frac{2}{\mu(\gamma+t)}$. Then the convergence rate of FedAvg by using MUCSC for model update is:\begin{equation*}
    \mathbb{E}\left[F(\mathbf{w}_{T})\right]-F^*\le\frac{2\kappa}{\gamma+T}(\frac{\tau}{\mu}+2L\left\|\bar{\mathbf{w}}_{0}-\mathbf{w}^*\right\|^2),
\end{equation*}
where $\tau=\sum_{i=1}^Np_i^2\sigma_i^2+6L\Gamma+8(E-1)^2G^2+(\frac{d^2+2d(Z_D-1)^2}{(Z_D-1)^4}+\frac{2d}{(Z_U-1)^2})E^2G^2\sum_{i=1}^Np_i^2$.
\end{theorem}
Here $Z_U$ and $Z_D$ are the numbers of centroids for uplink and downlink compression respectively.
Please refer to Appendix \ref{ProofOfTheorem1And2} for the detailed proof.

{\bf Remark:}
We can conclude that the convergence rate with MUCSC for model update compression is $O(\frac{1}{T})$, which is the same as that derived in \cite{li2019convergence}.\footnote{Here we regard $E$ as a constant number. Tuning $E$ may alter the convergence rate. However, this problem is out of the scope of our paper. } However, the compression error can still affect the model accuracy since $\tau$ can be inflated by compression errors. Note that what we present here is the upper bound of the convergence rate. The actual influence can be much smaller through minimizing compression errors. 
It is also interesting to note the tradeoff between the model accuracy and the compression rate. Intuitively, to achieve a higher compression rate, we should set smaller $Z_U$ and $Z_D$, which however makes $\tau$ larger.

\subsection{Partial client participation in training}

In reality, due to the restriction of the limited network resource, it may not be possible to involve the complete set of clients in each global iteration for model updating. A commonly adopted approach is to randomly select a number of $K$ clients to conduct local iterations.
Thus, it is more meaningful to analyze the convergence rate under the partial client participation mode. 
\begin{assumption}\label{Assump:SelectScheme} (Partial Client Participation Mode)
In each round of global iteration, the PS distributes the latest model updates to $N$ clients, but only randomly select $K$ clients  according to the weight probabilities $p_1, p_2, \dots, p_N$ with replacement  to conduct local iterations. The set of selected clients is denoted by $\mathcal{K}_t$ at the $t^{th}$ global iteration. 
\end{assumption}

According to Assumption \ref{Assump:SelectScheme}, the PS aggregates model updates by $\mathbf{w}_t = \mathbf{w}_{t-E} - \widetilde{\mathbf{D}}_t$ where $\mathbf{D}_t =  \frac{1}{K}\sum_{i\in\mathcal{K}_t}\widetilde{\mathbf{U}}_{t}$ and $\widetilde{\mathbf{D}}_t$ is the compressed result of $\mathbf{D}_t$.  According to the derivation of \cite{li2019convergence}, we know that such random client participation mode generates unbiased estimation. 
In order to synchronize all clients, the PS needs to distribute aggregated model updates to all clients, though only $K$ out of $N$ clients conduct local iterations. We  argue that this synchronization is necessary without incurring much communication overhead in FL. Firstly, FL is an open platform and clients can depart the system at any time. Synchronization can make clients have the latest models at any time when they depart the system. Secondly, the downlink capacity is usually much larger than uplink capacity for modern Internet such that the broadcasting of aggregated model updates is very efficient \cite{speedtest.net}. Lastly, the MUCSC algorithm can reduce the volume of the downlink traffic significantly, which will not consume much network resource.


Under the partial participation mode, we need to revise the iterative formula a little bit.
When $t+1\in\mathcal{I}$, we define $t_0=t+1-E$ and use $\mathcal{K}_{t+1}=\{i_1, i_2, \dots i_K\}$ to denote the set of clients selected for training. Therefore, the update process of the model can be defined as
\begin{equation}
    \label{EQ:LocalUpdate}
    \begin{split}
        \mathbf{v}^i_{t+1}=\mathbf{w}^i_{t}-\eta_t\nabla F_i(\mathbf{w}^i_t,\mathcal{B}^i_t).
    \end{split}
\end{equation}
\begin{equation}
    \label{EQ:GlobalUpdateOfCompAndPartial}
    \mathbf{w}^i_{t+1}=\left\{
        \begin{aligned}
        &\mathbf{v}^i_{t+1}, \quad if\ t+1\ \notin \mathcal{I},\\
        &\mathbf{w}_{t_0}^i - \widetilde{\mathbf{D}}_{t+1}, \quad if\ t+1\ \in \mathcal{I}, \\
        \end{aligned}
        \right.
\end{equation}
where $\mathbf{D}_{t+1}=\frac{1}{K}\sum_{i\in\mathcal{K}_{t+1}}\widetilde{\mathbf{U}}^i_{t+1}$ and $\mathbf{U}^i_t=\sum_{j=t_0}^{t} \eta_{j} \nabla F_i(\mathbf{w}^i_{j},\mathcal{B}^i_{j})$.

 {Under the partial participation mode,  the convergence rate analysis is derived as below.}

\begin{theorem}
\label{Theorem3}
Let Assumption \ref{Assump:ConSmoo}-\ref{Assump:BoundG} hold. For a decreasing learning rate, we define $\eta_t=\frac{\beta}{t+\gamma}$, where $\beta>\frac{1}{\mu}$ and $\gamma>0$ such that $\eta_1\leq\min\{\frac{1}{\mu}, \frac{1}{4L}\}$ and $\eta_t\leq2\eta_{t+E}$. Then the convergence rate of FedAvg by using MUCSC for model update is:
\begin{equation*}
    \mathbb{E}\left[F(\mathbf{w}_{T})\right]-F^*\leq \frac{L}{2}\frac{v_T}{T+\gamma}, 
\end{equation*}
where $v_T=\max\{\frac{\beta^2\tau}{\beta\mu-1}+\sum_{t=1}^T[(t+\gamma)\psi_t], \gamma\|\mathbf{w}_0-\mathbf{w}^*\|^2\}$, $\tau=\frac{\sum_{i=1}^Np_i^2\sigma_i^2}{B}+6L\Gamma+8(E-1)^2G^2+\frac{4E^2G^2}{K}$ and $\psi_{t+1}=J^D_{t+1} + \frac{1}{K}\sum_{i=1}^Np_{i} J^{i}_{t+1}$.
\end{theorem}

Please refer to Appendix \ref{ProofOfTheorem3And4} for the detailed proof.
Again, through Theorem \ref{Theorem3}, we can observe that the compression loss functions, \emph{i.e.}, $J^D_{t+1}$ and $J^{i}_{t+1}$, affect the convergence as well under the partial participation mode. 
Therefore, it is reasonable to minimize $J^D_{t+1}$ and $J^{i}_{t+1}$ through the MUCSC algorithm so that  the influence of compression errors can be minimized

Through bounding $J^D_{t+1}$ and $J^i_{t+1}$, we can also derive the convergence rate as below. 
\begin{theorem}
\label{Theorem4}
Let Assumptions \ref{Assump:ConSmoo}-\ref{Assump:BoundG} hold. Let $\kappa=\frac{L}{\mu},\gamma=\max\{8\kappa, E\}$ and $\eta_t=\frac{2}{\mu(\gamma+t)}$. Then the convergence rate of FedAvg by using MUCSC for model update is:
\begin{equation*}
    \mathbb{E}\left[F(\mathbf{w}_{T})\right]-F^*\le\frac{2\kappa}{\gamma+T}(\frac{\tau}{\mu}+2L\left\|\bar{\mathbf{w}}_{0}-\mathbf{w}^*\right\|^2),
\end{equation*}
where $\tau=\sum_{i=1}^Np_i^2\sigma_i^2+6L\Gamma+8(E-1)^2G^2+\frac{4}{K}E^2G^2+(\frac{d^2+2d(Z_D-1)^2}{(Z_D-1)^4}+\frac{2d}{K(Z_U-1)^2})E^2G^2$.
\end{theorem}

Please refer to Appendix \ref{ProofOfTheorem3And4} for the detailed proof.
This convergence rate is $O(\frac{1}{T})$, which is also the same as that of the FedAvg algorithm without compression derived in \cite{li2019convergence}.  We can also observe the tradeoff between the model accuracy and the compression rate here, and  that $    \mathbb{E}\left[F(\mathbf{w}_{T})\right]$converges to $F^*$  finally as $T$ approaches infinity.

 {Although MUCSC incurs compression errors, the intuition that MUCSC will not lower the converge rate can be explained as below. MUCSC conducts unbiased estimation of original parameters, and the estimation error can be strictly bounded. With the decrease of the learning rate, the influence of the estimation error also diminishes, and hence the FL training converges. Similar to the role of the variance of the stochastic gradients in each client in Assumption~\ref{Assump:LocalVar}, the compression error only increases the variance of learned parameters, and will not lower the convergence rate. 
}

\subsection{Tradeoff between Model Accuracy and Compression Rate}

From Theorems \ref{Theorem2}
and \ref{Theorem4}, we can assert that the model accuracy will be lower if we set smaller $Z_U$ and $Z_D$. Now, we explicitly derive the compression rate based on $Z_U$ and $Z_D$ to further reveal the tradeoff between the model accuracy and the compression rate. 

Suppose the dimension of the trained model is $d$ and the number of centroids for uplink and downlink compression operations are $Z_U$ and $Z_D$ respectively.  We assume that the size of a centriod or an original parameter is $h$ bytes.   Then, we can calculate the compression rate, denoted by $\Lambda$, for a particular client as follows. 
It takes $\frac{K}{N}hZ_U+hZ_D$ bytes to transmit all centroid values for uplink and downlink between the PS and a single client in a round of global iteration. Here $\frac{K}{N}$ is the probability that each  client is involved to conduct local iterations. 
It takes $\frac{\left\lceil\log_2{Z_U}\right\rceil}{8}/\frac{\left\lceil\log_2{Z_D}\right\rceil}{8}$ bytes to upload/download a centroid ID. 
It turns out that the compression rate is:
\begin{equation}
\label{CompressionRateofMUCSC}
    \begin{split}
        \Lambda &= \frac{2hd}{(\frac{\left\lceil\log_2{Z_U}\right\rceil}{8}\times d+hZ_U)\frac{K}{N}+hZ_D+ \frac{\left\lceil\log_2{Z_D}\right\rceil}{8}\times d} \\
        &< \frac{16h}{\left\lceil\log_2{Z_U}\right\rceil \frac{K}{N} + \left\lceil\log_2{Z_D}\right\rceil},
    \end{split}
\end{equation}
where $Z_U\geq 2$ and $Z_D\geq 2$. If $d\gg Z_U$ and $d \gg Z_D$, we have $\Lambda \approx \frac{16h}{\left\lceil{\log_2{Z_U}}\right\rceil \frac{K}{N} +\left\lceil{\log_2{Z_D}}\right\rceil}$.
From \eqref{CompressionRateofMUCSC}, we can see that the compression rate can be improved if we reduce the number of centroids $Z_U$ and $Z_D$. However, recall  that $J^i_t$ and $J^D_t$ defined in 
\eqref{EQ:VarianceOfCompression} will increase with the decrease of  $Z_U$ and $Z_D$, and hence the model accuracy will be lowered. 

\subsection{Boosted MUCSC}

According to \eqref{CompressionRateofMUCSC}, the highest compression rate $\Lambda$  is $\frac{16h}{\frac{K}{N}+1}$ which can be achieved by setting $Z_U=2$ and $Z_D=2$. However, setting $Z_U=2$ and $Z_D=2$ will result in maximized  compression errors defined in \eqref{EQ:VarianceOfCompression}, and hence lower the model accuracy.  To further improve the compression rate without compromising model accuracy significantly, we propose the Boosted MUCSC (B-MUCSC) by incorporating the Deep Gradient Compression (DGC) algorithm into MUCSC. DGC is proposed in  \cite{lin2017deep} based on the fact that the values of the most model updates are close to 0 \cite{aji2017sparse}. Thus, the trained model can be effectively updated by only accurately uploading a very small fraction (\emph{e.g.}, 1\%) of model updates which are far away from 0. 

The B-MUCSC algorithm works as follows. Given a vector $\mathbf{U}$, rank $U_{m}$'s  by decreasing order of their absolute values and select top $d_0$ model updates. Then, apply the MUCSC algorithm to cluster $d_0$ elements into $Z_U$ clusters. The rest $d-d_0$ model updates will be simply estimated  by their average value.\footnote{ $d_0$ is a tuneable hyper-parameter.  
In our experiments, it is set as 0.01d for both uplink and downlink model update compression.  }  Downlink model updates can be compressed in a similar way. 

For the B-MUCSC algorithm, each client needs to upload $(\frac{\left\lceil\log_2{Z_U}\right\rceil+\left\lceil\log_2{d}\right\rceil}{8}\times d_0 + hZ_U + h)\frac{K}{N}$ bytes. Here we need $\frac{\left\lceil\log_2{Z_U}\right\rceil+\left\lceil\log_2{d}\right\rceil}{8}d_0$ bytes to represent both cluster IDs and parameter IDs for $d_0$ parameters, and $h$ bytes for the average value of $d-d_0$ parameters. Similarly, each client needs to download $\frac{\left\lceil\log_2{Z_D}\right\rceil+\left\lceil\log_2{d}\right\rceil}{8}\times d_0 + hZ_D + h$ bytes from the PS. The compression rate of B-MUCSC, denoted by $\Lambda_{B}$, is:
\begin{equation}
\label{EQ:CompressRateB}
    \begin{split}
        \Lambda_{B} &= \frac{2hd}{\omega_U+\omega_D},\\
        \omega_U &= (\frac{\left\lceil\log_2{Z_U}\right\rceil+\left\lceil\log_2{d}\right\rceil}{8}d_0 + hZ_U +h)\frac{K}{N},\\
        \omega_D &= \frac{\left\lceil\log_2{Z_D}\right\rceil+\left\lceil\log_2{d}\right\rceil}{8}d_0+h(Z_D+1).
    \end{split}
\end{equation}
If $d\gg d_0\gg Z_U, Z_D$, we have $\Lambda_B \approx \frac{16hd}{(\left\lceil\log_2{Z_U}\right\rceil+\left\lceil\log_2{d}\right\rceil)d_0\frac{K}{N}+(\left\lceil\log_2{Z_D}\right\rceil+\left\lceil\log_2{d}\right\rceil)d_0}\gg \Lambda$. Thus, the compression rate can be increased considerably even if $Z_U$ and $Z_D$ are set as  relatively large values. 

{ 
\subsection{Practical Implementation}

The Internet traffic reduced  by the MUCSC algorithm comes at the cost of more computation cost. The practical system is rather complicated, and it is not reasonable to merely consider the compression rate (controlled by $Z_u$ and $Z_D$) in practice. We briefly the factors to influence the choice of  $Z_u$ and $Z_D$ in real systems. 

\begin{itemize}
    \item Downlink/Uplink capacity: If the downlink/uplink capacity is very limited, a small $Z_D$/$Z_U$ should be chosen to achieve a high compression rate so as to reduce the downlink/uplink traffic. 
    \item Computation capacity of the  PS/clients: If the computation capacity of the PS/clients is limited, a small $Z_D$/$Z_U$ should  be chosen to reduce the computation complexity of MUCSC. 
\end{itemize}

In summary, $Z_U$ and $Z_D$ can tradeoff the compression rate, the computation resource consumption and the model accuracy. MUCSC is particularly applicable if the computation capacity (of the PS and clients) is excessive but the communication capacity is limited. However, it is very difficult to theoretically incorporate the restrictions of computation and communication capacity into our convergence rate analysis.  Thus, we propose to regard $Z_U$ and $Z_D$ as hyperparameters, which can be determined based on  empirical experience. In the next section, we will conduct experiments to demonstrate the benefits of MUCSC by easily setting reasonable $Z_U$ and $Z_D$.

}


\section{Performance Evaluation} \label{Performance}

In this section, we evaluate the performance of the MUCSC and B-MUCSC algorithms from multiple aspects. 


\subsection{Experimental Settings}

\subsubsection{Dataset} 
We use CIFAR-10 \cite{CIFAR10} and FEMNIST \cite{caldas2018leaf} for our experiments. 
The CIFAR-10 dataset includes 60,000 32*32 color images which can be classified into 10 classes and each class contains 6,000 images. 
We randomly select 50,000 images (\emph{i.e.}, 5,000 from each class) as the training set and the rest 10,000 images are used  as the test set.  

FEMNIST is a dataset specially designed for FL, and each picture in FEMNIST is a 28*28 grayscale image. FEMNIST contains a total number of 62 categories, including numbers from 0-9 and all uppercase and lowercase English letters generated by individual users. 
We generate $100$ clients and the local dataset for each client by using the method given in \cite{caldas2018leaf}.  In total, there are 41,761 samples, and each client owns about 400 samples. We randomly select 90\% samples from each client as the training set. The rest samples are used as the test set. 


\subsubsection{Trained Models}
A full convolutional layer network (called ALL-CNN) is trained to classify the CIFAR-10 dataset. The model is designed based on the  work \cite{springenberg2014striving}. Specifically, the model consists of 9 convolutional layers, with a total number of more than $10^6$ parameters.  
In our work, each client uses the Stochastic Gradient Descent (SGD) method to perform local iterations with the learning rate $\max\{\frac{0.5}{1+\frac{t}{400}}, 0.01\}$ where $t$ is the iteration index. 

For experiments with the FEMNIST dataset, we implement another convolutional neural network (CNN) improved based on the LeNet-5 network in \cite{lecun1998gradient}. We also employ the SGD algorithm to train the model and the setting of the learning rate is the same as that of ALL-CNN.

\subsubsection{System Settings}

To evaluate model update compression algorithms in a realistic scenario, we simulate   an FL system with a single PS and 100 clients. 

 {In our experiment, we adopted a star network topology, that is, all clients are connected to the PS with a point-to-point mode.  The star topology is the most popular one used in existing works \cite{nishio2019client, sattler2019robust}.\footnote{Different topological structures will affect the convergence rate of FL.
However, the compression rate of MUCSC is independent with the network topology. Due to limited space, we only adopt the star topology in this work. }
}

In each round of global iteration, 10 clients are randomly selected to upload their model updates  for model aggregation at the PS. The network is setup according to the work \cite{akdeniz2014millimeter}, in which there is  a single PS  co-located with a bases station. Clients are evenly distributed around the PS in a circle of 2km. The model updates are transmitted between the PS and each client via wireless communication channels with the carrier frequency 2.5GHz. The antenna lengths of the base station and each client are 11m and 1m, respectively. We set the transmission power and antenna gain of the PS and each client as 20dBm and 0dBi, respectively. 
The average communication speed  is set as 1.4 Mb/s for uplink and downlink.  However, the actual communication speed between the PS and each client is  a random variable sampled from the Gaussian distribution  at the beginning of each global iteration. We set the standard deviation of the Gaussian distribution as 10\% of the mean value. 
With our settings, the communication time consumed by each global iteration always depends on the slowest client. For example, if the download speed of client $n$ is $\theta_n$ and the client needs to download $M$ bytes from the PS. Then, the time for $N$ clients to complete the download of model updates is $\max_n\{\frac{M}{\theta_n}\}$.


In our experiments, we set the batch  size as 8 for local iterations. The number of local iterations on each selected client is $E=5$.  For the MUCSC algorithm, we set $Z_D=16$, and $Z_U$ as  $4$, $8$ or $16$ by randomly splitting clients into 3 sets. The purpose to set different $Z_U$'s is to evaluate the robustness of MUCSC, which can work well even if clients use different number of centriods. 

 {The compression rate of uplink/downlink compression is mainly determined by the hyperparameter $Z_U$/$Z_D$. We suppose that the PS has more abundant communication resources and computing resources than clients, and hence $Z_D$ can be set as a larger value. For clients, we consider a general heterogeneous scenario, and there exist both fat clients with abundant resources and thin clients with limited resources. Thus, we set $Z_U$ as different values. In our settings, the computed uplink/downlink compression rate is  10.67/8.0.}

For the B-MUCSC algorithm, we select the top 1\% most significant model updates for compression.  Since only a very small fraction of model updates are involved for compression, we simply set  $Z_U=Z_D=256$.
We initialize the  learning rate of the MUCSC/B-MUCSC algorithm  as 0.001. It will iterate 5 times. In each iteration, the learning rate will be decreased by 10 times if any newly generate  centroid value is out of the range of model updates. 

\subsubsection{Sample Distribution}
We evaluate our algorithms with both iid and non-iid sample distributions. 


\begin{itemize}
    \item IID Distribution: The IID distribution is setup by using the CIFAR-10 dataset. Each client  randomly and uniformly selects 500 samples from the entire training set as the local dataset.
    \item Non-IID Distribution: For the non-iid distribution with CIFAR-10, each client randomly selects 300 to 500 samples from the training set, but each client can only have samples from  5 randomly selected classes. 
    For the FEMNIST dataset, its distribution is naturally non-iid. Each client can be regarded as a separate writer, with samples generated by herself. The number of data samples owned by  each client is about 400.
\end{itemize}
 
\subsubsection{Evaluation Metrics}
The evaluation of the compression algorithm is rather complicated. It is unreasonable to use a single metric to evaluate a compression algorithm in FL. For example, a biased compression algorithm can achieve an extremely high compression rate. However, the model accuracy may be deteriorated  significantly by the compression error. 
In our work, we adopt the following five metrics for evaluation from various aspects. 
\begin{itemize}
    \item Model Accuracy: Evaluate whether a compression algorithm will compromise the trained model accuracy evaluated based on the test set 
    \item  {Compression Rate: Evaluate the ratio of uncompressed communication traffic to compressed communication traffic in a single round of global iteration  for each compression algorithm.}
    \item  {Total Communication Traffic: Evaluate the total volume of the  communication traffic to complete the whole FL training process. Note that different algorithms may take different number of global iterations to complete the FL training so that the model can reach the specified accuracy. This metric is the product of the average traffic in each global iteration and the total number of global iterations for each algorithm. }
    \item Communication Time: Evaluate the total consumed communication time to complete the model training by adopting each compression algorithm. 
    \item  Computing Time: Evaluate the computing time spent by each client during  the entire model training process. 
\end{itemize}

\subsubsection{Baseline Algorithms}

We also implement several state-of-the-art compression algorithms as baselines. 
\begin{itemize}
    \item SIGNSGD (SSGD): SSGD quantifies model updates to be represented using  their signs. Thus, each model update value can be represented by a bit for  transmission \cite{bernstein2018signsgd}. 
    
    \item Sparse Ternary Compression (STC): STC will select top $x_{STC}\%$ model updates which possesses the largest absolute values. The selected model updates will be binarized and uploaded to the PS. The rest model updates will be set as 0. At the same time, clients keep model updates that are not uploaded \cite{sattler2019robust}. 
    
    \item Quantized SGD (QSGD): QSGD quantifies each model update by rounding it to $b$ discrete values in a principled manner that maintains the statistical properties of the original data \cite{alistarh2017qsgd}. 
    
    \item Deep Gradient Compression (DGC): The DGC algorithm only uploads top $x_{DGC}\%$\% model updates to the PS. For the remaining model updates, the algorithm will keep them locally on the client side and accumulate them in the next update \cite{lin2017deep}.
    
    \item No compression (NC): There is no compression operation. It is used as the benchmark for performance evaluation. It can also be regarded as a special compression algorithm with compression rate equal to 1. 
\end{itemize}
We set $x_{STC}\% = 3\%$ and  $x_{DGC}\% = 1\%$ in our experiments through tuning so that their overall performance is maximized.
SSGD and STC can be used for both uplink and downlink compression; while QSGD and DGC are developed to only compress uplink model updates. 

Among these baselines, only QSGD is an unbiased compression algorithm. All others are biased compression algorithms. Usually, biased algorithms can achieve higher compression rates with the cost of lower model accuracy.
According to \cite{shi2019convergence}, there is another way to classify compression algorithms. The  algorithms that only upload a small fraction of (compressed) model updates to the PS in each global iteration are regarded as the sparsification algorithms. The algorithms that quantify model updates into a small number of bits with error compensation are  regarded as quantization algorithms.   MUCSC, SSGD and QSGD are quantization algorithms, while B-MUCSC, STC and DGC are sparsification algorithms. 


\subsection{Experimental Results}

\subsubsection{Comparison of Model Accuracy}

\begin{figure}[h]
    \centering
    \includegraphics[width=\linewidth]{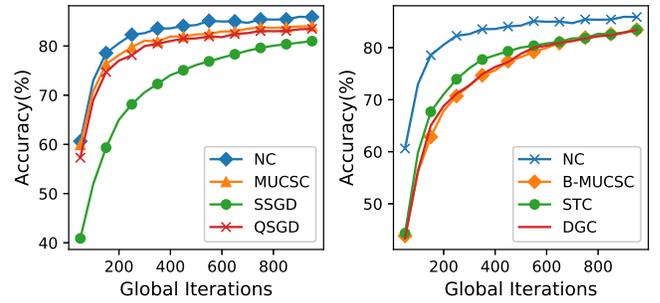}
    \caption{Comparison of model accuracy using different compression algorithms in FedAvg with iid data distribution of CIFAR-10: Quantization (left); Sparsification (right)}
    \label{CIFAR_IID_Accu}
\end{figure}

\begin{figure}[h]
    \centering
    \includegraphics[width=\linewidth]{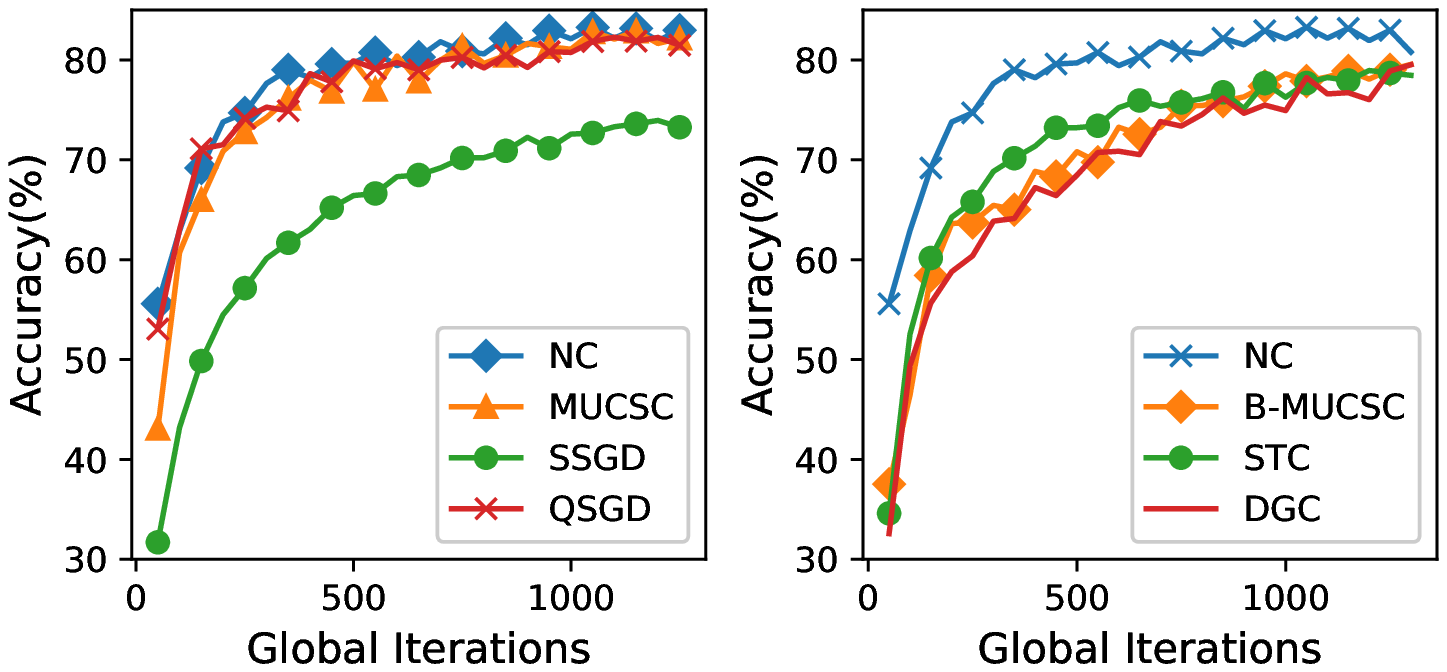}
    \caption{Comparison of model accuracy using different compression algorithms in FedAvg with non-iid data distribution of CIFAR-10: Quantization (left); Sparsification(right)}
    \label{CIFAR_NonIID_Accu}
\end{figure}

\begin{figure}[h]
    \centering
    \includegraphics[width=\linewidth]{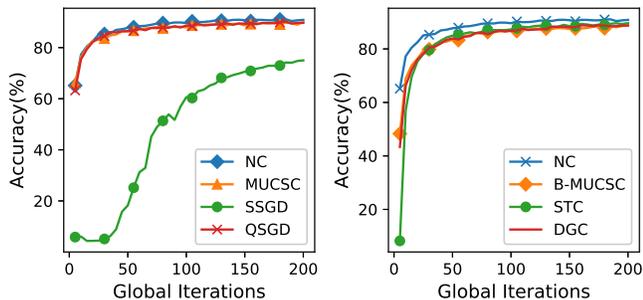}
    \caption{Comparison of model accuracy using different compression algorithms in FedAvg with non-iid data distribution of FEMNIST: Quantization (left); Sparsification(right)}
    \label{FEMNIST_NonIID_Accu}
\end{figure}

We conduct a series of experiments to compare the model accuracy by executing a certain number of iterations with FedAvg and different compression algorithms. Our results are presented in Fig.~\ref{CIFAR_IID_Accu} (CIFAR-10 with iid), Fig.~\ref{CIFAR_NonIID_Accu} (CIFAR-10 with non-iid) and Fig.~\ref{FEMNIST_NonIID_Accu} (FEMNIST with non-iid).  To avoid plotting too many curves in a single figure, we plot the results of quantization algorithms and sparsification algorithms in separate subfigures.  
In all figures, the x-axis represents the number of conducted global iteration; while the y-axis represents the model accuracy evaluated by the test set. 

From the experiment results, we can observe that
\begin{itemize}
    \item Because MUCSC and QSGD are unbiased compression algorithms, the model accuracy achieved by using them is very close to that of NC, and better than other baselines. 
    \item Although SSGD is a quantization algorithm, its accuracy performance is very poor because it is a biased compression algorithm. 
    \item There is a significant gap of the model accuracy between NC and biased algorithms, \emph{i.e.}, B-MUCSC, STC and DGC. The reason is that biased algorithms  incur larger compression errors. 
    \item Through comparing Fig.~\ref{CIFAR_IID_Accu} with Fig.~\ref{CIFAR_NonIID_Accu}, we can see that the non-iid sample distribution lowers the model accuracy a little bit. 
\end{itemize}

\subsubsection{Comparison of Compression Rates}

\begin{table}[h]
    \centering
    \caption{Comparison of compression rates of algorithms in CIFAR-10(second line) and FEMNIST(third line).}
    \begin{tabular}{|c|c|c|c|c|c|}
        \hline
        \multicolumn{1}{|c|}{MUCSC} & \multicolumn{1}{|c|}{QSGD} & \multicolumn{1}{|c|}{SSGD} & \multicolumn{1}{|c|}{B-MUCSC} & \multicolumn{1}{|c|}{DGC} & \multicolumn{1}{|c|}{STC}\\
        \hline
        $14.54\times$ & $1.98\times$ & $58.18\times$ & $196.57\times$ & $1.99\times$ & $88.15\times$\\
        \hline
        $14.54\times$ & $1.98\times$ & $58.18\times$ & $196.53\times$ & $1.99\times$ & $88.15\times$\\
        \hline
    \end{tabular}
    \label{CompressionRateCIFARFEMNIST}
\end{table}

As we have stated that it is not reasonable to only compare compression algorithms in terms of model accuracy. 
We further compare compression rates achieved by compression algorithms in experiments presented in Fig.~\ref{CIFAR_IID_Accu}, Fig.~\ref{CIFAR_NonIID_Accu} and Fig.~\ref{FEMNIST_NonIID_Accu}.   For MUCSC and B-MUCSC, the compression rates are computed according to \eqref{CompressionRateofMUCSC} and \eqref{EQ:CompressRateB} respectively. The compression rates for other baselines can be computed in a similar principle:  the total traffic volume without any compression divided by the total traffic volume after compression. 
The compression rates are listed in Tables~\ref{CompressionRateCIFARFEMNIST}. Note that in Table~\ref{CompressionRateCIFARFEMNIST}, the second line corresponds to Figs.~\ref{CIFAR_IID_Accu} and \ref{CIFAR_NonIID_Accu} since the sample distribution does not change the compression rate. 
Through comparing compression rates,  we can see that:
\begin{itemize}
    \item B-MUCSC  significantly outperforms other compression algorithms in terms of the compression rate. In particular, the results in Tables~\ref{CompressionRateCIFARFEMNIST} indicate that B-MUCSC can reduce the communication traffic by more than 99\% per communication round. 
    \item  {There is a tradeoff between the model accuracy and the compression rate. Although the model accuracy of STC is a little bit better than that of B-MUCSC, its compression rate is much lower than that of  B-MUCSC. }
    \item The compression rate of MUCSC is much higher than QSGD, another unbiased compression algorithm.\footnote{Recall that it is unfair to compare the compression rate between biased and unbiased algorithms. } 
    \item The compression rate with DGC is the lowest one among biased algorithms because it is only applicable to  compress uplink model updates. 
\end{itemize}

\subsubsection{Selection of Number of Centroids}

\begin{figure}[h]
	\centering
	\includegraphics[width=\linewidth]{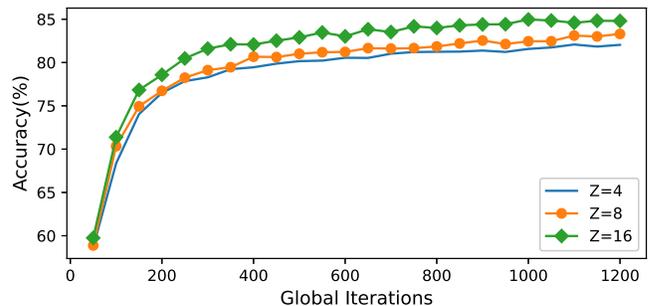}
	\caption{ {The impact of different numbers of centroids, \emph{i.e.}, Z, on model accuracy} }
	\label{DifferentZ}
\end{figure}

 {As we have discussed, there is a tradeoff between the model accuracy and the compression rate by varying the number of centriods $Z$ in MUCSC. To verify this point, we enumerate $Z$ as 4, 8 and 16 in the experiment presented in Fig.~\ref{DifferentZ}. We plot the model accuracy against the number of global iterations. The result in Fig.~\ref{DifferentZ} shows that the model accuracy will be better if $Z$ is larger. But the larger $Z$, the lower the compression rate. Thereby, to evaluate a compression algorithm in FL, one needs to synthetically consider both model accuracy and compression rate. }

\subsubsection{Comparison of Total Communication Traffic}

\begin{table}[h]
    \centering
    \caption{Comparison of communication traffic for different algorithms with iid distribution and 82\% model accuracy in CIFAR-10.}
    \begin{tabular}{|c|c|c|c|}
        \hline
         & T & Comm. Traffic & \tabincell{c}{comm. traffic\\reduced to}\\
         \hline
         NC & 1000 & 2090.05MB & $100\%$\\
         \hline
         \textbf{MUCSC} & \textbf{2250} & \textbf{587.88MB} & \textbf{28.13\%}\\
         \hline
         QSGD & 3000 & 3527.0MB & $168.75\%$\\
         \hline
         \textbf{B-MUCSC} & \textbf{3750} & \textbf{72.49MB} &\textbf{3.47\%}\\
         \hline
         DGC & 3750 & 3983.76MB & $190.61\%$\\
         \hline
         STC & 3750 & 161.65MB & $7.73\%$\\
         \hline
    \end{tabular}
    \label{IIDCommunicationCIFAR}
\end{table}

\begin{table}[h]
    \centering
    \caption{Comparison of communication traffic for different algorithms with non-iid distribution and 79\% model accuracy in CIFAR-10. }
    \begin{tabular}{|c|c|c|c|}
        \hline
         & T & Comm. Traffic & \tabincell{c}{comm. traffic\\reduced to}\\
         \hline
         NC & 1500& 3135.08MB & $100\%$\\
         \hline
         \textbf{MUCSC} & \textbf{2250} & \textbf{587.87MB} & \textbf{18.75\%}\\
         \hline
         QSGD & 2250 & 2645.25MB & $ 84.37\%$\\
         \hline
         \textbf{B-MUCSC} & \textbf{6000} & \textbf{115.99MB} &\textbf{3.69\%}\\
         \hline
         DGC & 6250 & 6639.59MB & $211.78\%$\\
         \hline
         STC & 6250 & 269.42MB & $8.59\%$\\
         \hline
    \end{tabular}
    \label{NonIIDCommunicationCIFAR}
\end{table}

\begin{table}[h]
    \centering
    \caption{Comparison of communication traffic for different algorithms with non-iid distribution and 85\% model accuracy in FEMNIST. }
    \begin{tabular}{|c|c|c|c|}
        \hline
         & T & Comm. Traffic & \tabincell{c}{comm. traffic\\reduced to}\\
         \hline
         NC & 100 & 206.93MB & $100\%$\\
         \hline
         \textbf{MUCSC} & \textbf{175} & \textbf{45.27MB} & \textbf{21.88\%}\\
         \hline
         QSGD & 150 & 149.65MB & $84.38\%$\\
         \hline
         \textbf{B-MUCSC} & \textbf{300} & \textbf{5.74MB} &\textbf{2.78\%}\\
         \hline
         DGC & 325 & 341.83MB & $165.19\%$\\
         \hline
         STC & 275 & 11.73MB & $5.16\%$\\
         \hline
    \end{tabular}
    \label{NonIIDCommunicationFEMNIST}
\end{table}

\begin{figure}[h]
	\centering
	\includegraphics[width=\linewidth]{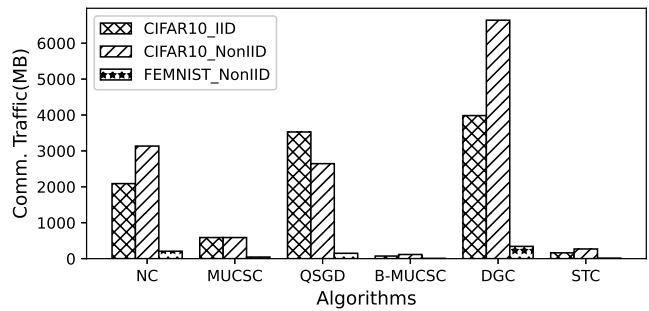}
	\caption{ {Comparison of communication traffic for different algorithms}}
	\label{CommTraffic}
\end{figure}

The total communication traffic is defined as the multiplication of communication rounds (\emph{i.e.}, the number of global iterations) and the volume of the communication traffic  of each global iteration. By leveraging this metric, we can evaluate the performance of each compression algorithm by setting a fixed model accuracy. In practice, this evaluation is more meaningful if the target of the model training is to obtain a model  with at least certain accuracy. 
Thus, we suppose that there is a predefined threshold as the target of model training.  The training process will not  terminate before  the model accuracy exceeds the threshold. With this setting, the numbers of conducted iterations are different if different compression algorithms are adopted. In Tables~\ref{IIDCommunicationCIFAR} (with CIFAR-10 and iid distribution), \ref{NonIIDCommunicationCIFAR} (with CIFAR-10 and non-iid distribution) and \ref{NonIIDCommunicationFEMNIST} (with FEMNIST and non-iid distribution), we  fix the target model accuracy as 82\%,79\% and 85\% respectively. 
Then, we compare the number of global iterations and the total communication traffic volume of each algorithm. Note that we cannot set a very high threshold in this experiment because biased compression algorithms often discard insignificant model updates such that  they  may never reach a very high accuracy.\footnote{SSGD is not included in this result because its model accuracy can never reach the threshold in our experiments. }
 {Fig.~\ref{CommTraffic} shows the communication traffic required for different algorithm  to obtain a predefined model accuracy. Note that the total number of global iterations conducted for each algorithm is different for this experiment. }
From Table~\ref{IIDCommunicationCIFAR}, we can draw the following observations:
\begin{enumerate}
    \item The unbiased algorithms (including NC, MUCSC and QSGD) take fewer numbers of iterations to reach the predefined threshold accuracy. However, they may consume more communication traffic because of their lower compression rate. 
    
    \item  B-MUCSC is the best algorithm that can significantly reduce the communication traffic by more than 95\% because it can optimally tradeoff between the model accuracy and the compression rate.  
    MUCSC is the best one among unbiased algorithms. 
    
    \item It is notable that the volume of the  communication traffic is increased to 211.78\% by using the DGC algorithm for compression. This counter-intuitive result can be explained from the increased number of global iterations. Although DGC reduces the traffic per global iteration, it needs to conduct a much larger number of global iterations to reach the threshold model accuracy, and thereby it finally generates more total traffic volume. 
\end{enumerate}


\subsubsection{Comparison of Total Time Cost}

\begin{figure}[h]
	\centering
	\includegraphics[width=\linewidth]{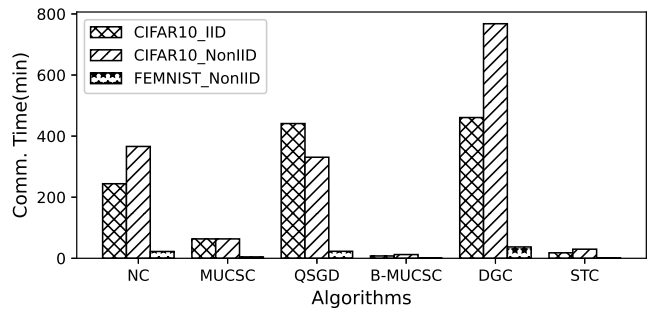}
	\caption{ {Comparison of communication time for different algorithms}}
	\label{CommTime}
\end{figure}

 {The advantage of compression algorithms is to reduce the communication time in FL. The comparison of the communication time required by different algorithms to achieve the predefined model accuracy is shown in Fig.~\ref{CommTime}.}

\begin{table}[h]
    \centering
    \caption{Comparison of total time cost for different algorithms with iid data distribution and 82\% model accuracy in CIFAR-10. }
    \begin{tabular}{|c|c|c|c|c|}
        \hline
         & T & Comm. Time & Comp. Time & \tabincell{c}{total time\\reduced to}\\
         \hline
         NC & 1000 & 244.1min & 47.57min & $100\%$\\
         \hline
         \textbf{MUCSC} & \textbf{2250} & \textbf{63.45min} & \textbf{153.52min} & \textbf{74.39\%}\\
         \hline
         QSGD & 3000 & 441.1min & 195.5min & $218.26\%$\\
         \hline
         \textbf{B-MUCSC} & \textbf{3750} & \textbf{7.62min} & \textbf{234.88min} &\textbf{83.14\%}\\
         \hline
         DGC & 3750 & 460.76min & 224.38min & $234.90\%$\\
         \hline
         STC & 3750 & 17.88min & 226.38min & $83.74\%$\\
         \hline
    \end{tabular}
    \label{IIDCommunicationTimeCIFAR}
\end{table}

\begin{table}[h]
    \centering
    \caption{Comparison of the total  time cost for different algorithms with non-iid data distribution and 79\% model accuracy in CIFAR-10. }
    \begin{tabular}{|c|c|c|c|c|}
        \hline
         & T & Comm. Time & Comp. Time & \tabincell{c}{total time\\reduced to}\\
         \hline
         NC & 1500 & 366.15min & 71.35min & $100\%$\\
         \hline
         \textbf{MUCSC} & \textbf{2250} & \textbf{63.45min} & \textbf{153.53min} & \textbf{49.59\%}\\
         \hline
         QSGD & 2250 & 330.83min & 146.63min & $109.13\%$\\
         \hline
         \textbf{B-MUCSC} & \textbf{6000} & \textbf{12.2min} & \textbf{375.8min} &\textbf{88.68\%}\\
         \hline
         DGC & 6250 & 767.92min & 373.96min & $261.00\%$\\
         \hline
         STC & 6250 & 29.81min & 377.29min & $93.05\%$\\
         \hline
    \end{tabular}
    \label{NonIIDCommunicationTimeCIFAR}
\end{table}

\begin{table}[h]
    \centering
    \caption{Comparison of the total time cost for different algorithms with non-iid data distribution and 85\% model accuracy in FEMNIST. }
    \begin{tabular}{|c|c|c|c|c|}
        \hline
         & T & Comm. Time & Comp. Time & \tabincell{c}{comm. time\\reduced to}\\
         \hline
         NC & 100 & 22.28min & 2.78min & $100\%$\\
         \hline
         \textbf{MUCSC} & \textbf{175} & \textbf{4.57min} & \textbf{8.62min} & \textbf{52.63\%}\\
         \hline
         QSGD & 175 & 22.88min & 7.8min & $122.43\%$\\
         \hline
         \textbf{B-MUCSC} & \textbf{300} & \textbf{0.63min} & \textbf{10.66min} &\textbf{45.05\%}\\
         \hline
         DGC & 325 & 37.15min & 10.67min & $190.82\%$\\
         \hline
         STC & 275 & 1.31min & 10.15min & $45.73\%$\\
         \hline
    \end{tabular}
    \label{CommunicationTimeFEMNIST}
\end{table}

From the entire system's perspective, less communication traffic volume does not necessarily result in shorter training time. The total training time cost is the sum of the communication time cost and the computation time cost. The computation time cost includes the time consumed by the model training and the model update compression.  To see the effectiveness of each compression algorithm in practice, we compare the total time cost of each experiment case to see how much time can be saved in Tables~\ref{IIDCommunicationTimeCIFAR}, \ref{NonIIDCommunicationTimeCIFAR} and  \ref{CommunicationTimeFEMNIST}. 
We keep the experiment settings and the threshold model accuracy the same as the last experiment. From results of this experiment, we can draw the following conclusions:
\begin{itemize}
    \item For NC, the communication time cost dominates the total time cost, and the computation time cost only takes about 20\% of the total time cost.  
     {\item The computation time of unbiased compression algorithms is much lower than that of biased algorithms. This is because unbiased compression algorithms take fewer communication rounds to achieve the target model accuracy. Hence, the cumulative computation time (\emph{i.e.} the multiplication of computation time  of each global iteration and  the total number of global iteration) of  unbiased algorithms is smaller.}
    \item QSGD and DGC cannot reduce the communication time cost effectively because of the significantly increased total number of global iterations. 
    \item MUCSC, B-MUCSC and STC can substantially reduce the communication time cost, however they all cause the increase of the computation time cost due to the compression operation. 
    \item MUCSC/B-MUCSC is the best one for CIFAR-10/FEMNIST, which can reduce the total time cost effectively. For CIFAR-10, the model dimension is more than $10^6$ so that the computation time cost is too high by using the B-MUCSC.\footnote{The computation time cost of B-MUCSC is higher because we set $Z_U=Z_D=256$ for B-MUCSC. } In contrast, the model dimension of the FEMNIST case is much lower, and thus the computation time cost is insignificant. B-MUCSC becomes the best one because it achieves  the lowest communication time cost. 
    \item Which compression algorithm is the best one also depends on the computation resource. A compression algorithm of higher computation complexity is more applicable if clients are equipped with more powerful computation resource. 
\end{itemize}



\section{Conclusion}\label{Conclusion}
Reducing the communication cost is an effective approach to accelerate the convergence of FL. In this work, we design a novel algorithm MUCSC by compressing both uplink and downlink high dimensional model updates into a few centroid values through a soft clustering algorithm. The amount of communication traffic between individual clients and PS can be significantly reduced since only  centroid values and centroid IDs are exchanged via the Internet. Meanwhile,  MUCSC can guarantee that the trained model will converge  with the same rate as that without compression. Our study also reveals the tradeoff between the compression rate and the model accuracy. A boosted MUCSC is devised by hiding IDs of parameters whose model updates are close to 0.  In the end, extensive experiments with the CIFAR-10 and FEMNIST datasets are carried out to verify our analysis and demonstrate the extraordinary performance of our compression algorithms.

\appendices
\section{Proof of convergence with full client participation}
We first consider the situation where all clients participate in training in each iteration. Therefore, the update process of the model is as follows.
\begin{equation}
    \label{EQ:LocalUpdate}
    \begin{split}
        \mathbf{v}^i_{t+1}=\mathbf{w}^i_{t}-\eta_t\nabla F_i(\mathbf{w}^i_t,\mathcal{B}^i_t).
    \end{split}
\end{equation}
\begin{equation}
    \mathbf{w}^i_{t+1}=\left\{
        \begin{aligned}
        &\mathbf{v}^i_{t+1}, \quad if\ t+1\ \notin \mathcal{I},\\
        &\mathbf{w}_{t+1-E}^i - \widetilde{\mathbf{D}}_{t+1}, \quad if\ t+1\ \in \mathcal{I}, \\
        \end{aligned}
        \right.
\end{equation}
where $\mathbf{D}_{t+1}=\sum_{i=1}^Np_i\widetilde{\mathbf{U}}^i_{t+1}$ and $\mathbf{U}^i_{t+1}=\sum_{j=t+1-E}^{t} \eta_{j} \nabla F_i(\mathbf{w}^i_{j},\mathcal{B}^i_{j})$

\subsection{Key Lemmas}
To facilitate our analysis, we  define two virtual sequences, namely $\bar{\mathbf{v}}_t=\sum_{i=1}^N p_i \mathbf{v}^i_{t}$ and $\bar{\mathbf{w}}_t=\sum_{i=1}^N p_i \mathbf{w}^i_{t}$ to represent the global average parameters. At the same time we define $\bar{\mathbf{g}}_t=\sum_{i=1}^N p_i \nabla F_i(\mathbf{w}^i_{t})$ and $\mathbf{g}_t=\sum_{i=1}^N p_i \nabla F_i(\mathbf{w}^i_{t},\mathcal{B}^i_t)$. By definition we can get $\bar{\mathbf{v}}_{t+1}=\bar{\mathbf{w}}_t-\eta_t \mathbf{g}_t$ and $\mathbb{E}[\mathbf{g}_t]=\bar{\mathbf{g}}_t$. 
In order to derive the convergence rate with the compressed model updates,  we need to leverage the lemmas that have been proved in \cite{li2019convergence}.

\begin{lemma}
	\label{Lemma:v-w}
	(Results of one step of iteration) Let Assumption \ref{Assump:ConSmoo} holds and $\eta_t\le\frac{1}{4L}$, we can get
	\begin{equation}
		\label{EQ:Lemma1}
		\begin{split}
			\mathbb{E}\left\|\bar{\mathbf{v}}_{t+1}-\mathbf{w}^*\right\|^2&=(1-\eta_t\mu)\mathbb{E}\left\|\bar{\mathbf{w}}_t-\mathbf{w}^*\right\|^2+\eta_t^2\mathbb{E}\left\|\mathbf{g}_t-\bar{\mathbf{g}}_t\right\|^2\\
			&\quad +6L\eta_t^2\Gamma+2\mathbb{E}\left[\sum_{i=1}^Np_i\left\|\bar{\mathbf{w}}_t-\mathbf{w}^i_t\right\|^2\right].
		\end{split}
	\end{equation}
\end{lemma}

\begin{lemma}
	\label{Lemma:VariOfg}
	(Bounding the variance) Let Assumption 2 holds. It follows that
	\begin{equation}
		\label{EQ:Lemma2}
		\mathbb{E}\left\|\mathbf{g}_t-\bar{\mathbf{g}}_t\right\|^2\le \frac{1}{B}\sum_{i=1}^Np_i^2\sigma_i^2.
	\end{equation}
\end{lemma}

\begin{lemma}
	\label{Lemma:LocalGandGloG}
	(Bounding the divergence of $\{\mathbf{w}^i_t\}$) Let Assumption \ref{Assump:BoundG} holds and $\eta_t$ is non-increasing and $\eta_t\le2\eta_{t+E}$ for all $t$, we have
	\begin{equation}
		\label{EQ:Lemma3}
		\mathbb{E}\left[\sum_{i=1}^Np_i\left\|\bar{\mathbf{w}}_t-\mathbf{w}^i_t\right\|^2\right]\le 4\eta_t^2(E-1)^2G^2.
	\end{equation}
\end{lemma}

Lemmas \ref{Lemma:v-w}-\ref{Lemma:LocalGandGloG} can be trivially obtained by extending the analysis of the FedAvg algorithm provided in \cite{li2019convergence}. 
However, the influence of the model update compression has not been considered by these lemmas. 

Since here, we introduce the lemmas to quantify the influence of the compression algorithm. 
According to the update process of $\bar{\mathbf{w}}_t$ and $\bar{\mathbf{v}}_t$, we can obtain the upper bound of the variance of $\bar{\mathbf{w}}_t$.
\begin{lemma}
	\label{Lemma:AllClientsComp}(Bounding the variance of $\bar{\mathbf{w}}_t$ cause by compression) Let Assumption \ref{Assump:ConSmoo}-\ref{Assump:BoundG} hold.  When $t+1 \in \mathcal{I}$, $\eta_t$ is non-increasing and $\eta_t\le2\eta_{t+E}$ for all $t$, then we have 
	\begin{equation}
		\label{EQ:Lemma4}
		\mathbb{E}\left\|\bar{\mathbf{v}}_{t+1}-\bar{\mathbf{w}}_{t+1}\right\|^2\le J^D_{t+1} + \sum_{i=1}^Np_i^2J^i_{t+1},
	\end{equation}
	where $J^D_{t+1}=\mathbb{E}\|\widetilde{\mathbf{D}}_{t+1} - \mathbf{D}_{t+1}\|^2$ and $J^i_{t+1}=\mathbb{E}\|\widetilde{\mathbf{U}}^i_{t+1} - \mathbf{U}^i_{t+1}\|^2$.
\end{lemma}
Please refer to Appendix \ref{ProofOfLemma4} for the detailed proof.

\subsection{Variance between $\bar{\mathbf{w}}_{t+1}$ and $\bar{\mathbf{v}}_{t+1}$} \label{ProofOfLemma4}
We first calculate the upper bound of the variance between $\bar{\mathbf{w}}_{t+1}$ and $\bar{\mathbf{v}}_{t+1}$ at this time.

\begin{equation}
    \label{WAndVFull}
    \begin{split}
        &\mathbb{E} \left\|\bar{\mathbf{w}}_{t+1}-\bar{\mathbf{v}}_{t+1}\right\|^2=\mathbb{E}\left\|\widetilde{\mathbf{D}}_{t+1} - \sum_{i=1}^Np_i\mathbf{U}^i_{t+1}\right\|^2\\
        &=\mathbb{E}\left\|\widetilde{\mathbf{D}}_{t+1} - \mathbf{D}_{t+1} + \sum_{i=1}^Np_i\widetilde{\mathbf{U}}^i_{t+1} - \sum_{i=1}^Np_i\mathbf{U}^i_{t+1}\right\|^2\\
        &=\mathbb{E}\left\|\widetilde{\mathbf{D}}_{t+1} - \mathbf{D}_{t+1}\right\|^2 + \mathbb{E}\left\|\sum_{i=1}^Np_i\widetilde{\mathbf{U}}^i_{t+1} - \sum_{i=1}^Np_i\mathbf{U}^i_{t+1}\right\|^2.
    \end{split}
\end{equation}

We use $J^D_{t+1}$ to represent the variance generated by compression of model updates aggregated on the server side and $J^i_{t+1}$ to represent the variance generated by compression of model updates to be uploaded by client $i$ at iteration $t+1$. Therefore
\begin{equation}
    \label{J_DFULL}
    \mathbb{E}\left\|\widetilde{\mathbf{D}}_{t+1} - \mathbf{D}_{t+1}\right\|^2 = J^D_{t+1},
\end{equation}
and
\begin{equation}
    \label{J_IFULL}
    \begin{split}
        \mathbb{E}\left\|\sum_{i=1}^Np_i\widetilde{\mathbf{U}}^i_{t+1} - \sum_{i=1}^Np_i\mathbf{U}^i_{t+1}\right\|^2 &\leq \sum_{i=1}^Np_i^2\mathbb{E}\left\|\widetilde{\mathbf{U}}^i_{t+1} - \mathbf{U}^i_{t+1}\right\|^2\\
        &\leq \sum_{i=1}^Np_i^2J^i_{t+1}.
    \end{split}
\end{equation}

We can combine \eqref{WAndVFull}, \eqref{J_DFULL} and \eqref{J_IFULL} to get 
\begin{eqnarray}
    \label{BoundOfWAndVFULL}
    \mathbb{E} \left\|\bar{\mathbf{w}}_{t+1}-\bar{\mathbf{v}}_{t+1}\right\|^2 \leq J^D_{t+1} + \sum_{i=1}^Np_i^2J^i_{t+1}.
\end{eqnarray}

\subsection{Convergence of the algorithm} \label{ProofOfTheorem1And2}
We next analyze the variance between $\bar{\mathbf{w}}_{t+1}$ and $\mathbf{w}^*$.
\begin{equation}
    \begin{split}
        &\mathbb{E}\left\|\bar{\mathbf{w}}_{t+1} - \mathbf{w}^*\right\|^2\\
        &=\mathbb{E}\left\|\bar{\mathbf{w}}_{t+1} - \bar{\mathbf{v}}_{t+1} + \bar{\mathbf{v}}_{t+1} - \mathbf{w}^*\right\|^2\\
        &=\mathbb{E}\left\|\bar{\mathbf{w}}_{t+1} - \bar{\mathbf{v}}_{t+1}\right\|^2 + \mathbb{E}\left\|\bar{\mathbf{v}}_{t+1} - \mathbf{w}^*\right\|^2 + \\
        &\qquad2\mathbb{E}<\bar{\mathbf{w}}_{t+1} - \bar{\mathbf{v}}_{t+1}, \bar{\mathbf{v}}_{t+1} - \mathbf{w}^*>.
    \end{split}
    \label{VarianceOfWALL}
\end{equation}

When $t+1\in\mathcal{I}$, considering $\mathbb{E}\bar{\mathbf{w}}_{t+1}=\bar{\mathbf{v}}_{t+1}$, the expected value of the third term of \eqref{VarianceOfWALL} is 0. Using the aforementioned Lemma \ref{Lemma:v-w}-\ref{Lemma:LocalGandGloG} and \eqref{BoundOfWAndVFULL}, we can get
\begin{equation}
    \label{AllClientGap}
    \begin{split}
        &\mathbb{E}\left\|\bar{\mathbf{w}}_{t+1} - \mathbf{w}^*\right\|^2\\
        &=\mathbb{E}\left\|\bar{\mathbf{w}}_{t+1} - \bar{\mathbf{v}}_{t+1}\right\|^2 + \mathbb{E}\left\|\bar{\mathbf{v}}_{t+1} - \mathbf{w}^*\right\|^2\\
        &\leq (1-\eta_t\mu)\mathbb{E}\left\|\bar{\mathbf{w}}_{t} - \mathbf{w}^*\right\|^2 + \eta_t^2\tau + \psi_{t+1},
    \end{split}
\end{equation}
where $\tau=\sum_{i=1}^Np_i^2\sigma_i^2+6L\Gamma+8(E-1)^2G^2$ and $\psi_{t+1} = J^D_{t+1} + \sum_{i=1}^Np_i^2J^i_{t+1}$.

We finally analyze the convergence of the algorithm. We defined $\Delta_{t+1}=\mathbb{E}\left\|\bar{\mathbf{w}}_{t+1}-\mathbf{w}^*\right\|^2$. According to \eqref{AllClientGap}, we can get
$$\Delta_{t+1}\leq(1-\eta_t\mu)\Delta_t + \eta_t^2\tau + \psi_{t+1}.$$

For a decreasing learning rate, we define $\eta_t=\frac{\beta}{t+\gamma}$, where $\beta>\frac{1}{\mu}$ and $\gamma>0$ such that $\eta_1\leq\min\{\frac{1}{\mu}, \frac{1}{4L}\}$ and $\eta_t\leq2\eta_{t+E}$.

We will prove that $\Delta_{t}\leq\frac{v_t}{\gamma+t}$, where $v_t=\max\{\frac{\beta^2\tau}{\beta\mu-1}+\sum_{i=1}^t[(i+\gamma)\psi_i], \gamma\Delta_0\}$ by induction. The inequality holds when t=0. Assuming that the inequality holds for some $t$. So
\begin{eqnarray}
        &&\Delta_{t+1}\le(1-\eta_t\mu)\Delta_t+\eta_t^2\tau+\psi_{t+1}\notag\\
        &&\leq (1-\frac{\beta\mu}{t+\gamma})\frac{v_t}{\gamma+t}+\frac{\beta^2\tau}{(t+\gamma)^2}+\psi_{t+1}\notag\\
        &&\le (1-\frac{\beta\mu}{t+\gamma})\left(\frac{1}{t+\gamma}\frac{\beta^2\tau}{\beta\mu-1}+\frac{\sum_{i=1}^t((i+\gamma)\psi_i)}{t+\gamma}\right)\notag\\
        &&\qquad +\frac{\beta^2\tau}{(t+\gamma)^2}+\psi_{t+1}\notag\\
        &&\le\frac{(t+\gamma-1)}{(t+\gamma)^2}\frac{\beta^2\tau}{\beta\mu-1}+ \left[\frac{\beta^2\tau}{(t+\gamma)^2}-\frac{(\beta\mu-1)}{(t+\gamma)^2}\frac{\beta^2\tau}{\beta\mu-1}\right]\notag\\
        &&\qquad +\frac{t+\gamma-\beta\mu}{t+\gamma}\frac{\sum_{i=1}^t((i+\gamma)\psi_i)}{t+\gamma}+\psi_{t+1}\notag\\
        &&\le\frac{1}{t+\gamma+1}\frac{\beta^2\tau}{\beta\mu-1}+\frac{t+\gamma-1}{(t+\gamma)^2}\sum_{i=1}^t((i+\gamma)\psi_i)+\psi_{t+1}\notag\\
        &&\le\frac{v_{t+1}}{\gamma+t+1}.
\end{eqnarray}

So we prove that $\Delta_{t}=\mathbb{E}\left\|\bar{\mathbf{w}}_t-\mathbf{w}^*\right\|^2\leq \frac{v_t}{\gamma+t}$. 
Considering the L-smooth and strong convexity of $F()$, we can get
\begin{equation}
    \label{AllConverge}
    \mathbb{E}[F(\bar{\mathbf{w}}_t)]-F(\mathbf{w}^*)\leq\frac{L}{2}\Delta_t\leq \frac{L}{2}\frac{v_t}{t+\gamma}.
\end{equation}



In order to make the convergence result more intuitive, we next derive the upper bound of $J^D_{t+1}$ and $J^i_{t+1}$. 

We define $\mathbf{U}$ and $\widetilde{\mathbf{U}}$ to denote the original vector and the compressed vector of dimension $d$ respectively. We initialize the centroid value uniformly when solving the centroid values, so we can use the variance of this situation as the upper bound.
\begin{equation}
    \begin{split}
        \mathbb{E}\|\widetilde{\mathbf{U}}-\mathbf{U}\|^2&= \sum_{m=1}^d (r_{z+1}-U_{m})(U_{m}-r_{z})\\
        &\le \sum_{m=1}^d \frac{(r_{z+1}-r_{z})^2}{4}\\
        &\le \sum_{m=1}^d \frac{\left(\frac{U_{max}-U_{min}}{Z-1}\right)^2}{4} \le \frac{d}{2(Z-1)^2}\|\mathbf{U}\|^2,
    \end{split}
    \label{EQ:BoundOfVari}
\end{equation}
where $Z$ is the number of centroid values.

We can analyze the upper bound of the model updates to be compressed by the client $i$ using Assumption~\ref{Assump:BoundG} as follows.
\begin{equation}
    \begin{split}
        \mathbb{E}\left\|\mathbf{U}^{i}_{t+1}\right\|^2&= \mathbb{E}\left\|\sum_{j=t_0}^{t} \eta_{j} \nabla F_{i}(\mathbf{w}^{i}_{j},\mathcal{B}^{i}_{j})\right\|^2\\
        &\le E\sum_{j=t_0}^{t}\mathbb{E}\left\|\eta_{j} \nabla F_{i}(\mathbf{w}^{i}_{j},\mathcal{B}^{i}_{j})\right\|^2\\
        &\le E^2\eta_{t_0}^2G^2 \le 4E^2\eta_{t}^2G^2.
    \end{split}
\end{equation}

Therefore

\begin{eqnarray}
\label{BoundOfJ_DFULL}
    &&J^D_{t+1} = \mathbb{E}\left\|\widetilde{\mathbf{D}} - \mathbf{D}\right\|^2\notag\\
    &&\le \frac{d}{2(Z_D-1)^2}\mathbb{E}\left\|\sum_{i=1}^Np_i\widetilde{\mathbf{U}}^i_{t+1}\right\|^2\notag\\
    &&\le \frac{d}{2(Z_D-1)^2}\sum_{i=1}^Np_i^2\mathbb{E}\left\|\widetilde{\mathbf{U}}^i_{t+1} - \mathbf{U}^i_{t+1} + \mathbf{U}^i_{t+1}\right\|^2\notag\\
    &&\le \frac{d}{2(Z_D-1)^2}\sum_{i=1}^Np_i^2\notag\\
    && \qquad\left(\mathbb{E}\left\|\widetilde{\mathbf{U}}^i_{t+1} - \mathbf{U}^i_{t+1}\right\|^2+\mathbb{E}\left\| \mathbf{U}^i_{t+1}\right\|^2\right)\notag\\
    &&\le \frac{d^2+2d(Z_D-1)^2}{4(Z_D-1)^4}\sum_{i=1}^Np_i^2\mathbb{E}\left\| \mathbf{U}^i_{t+1}\right\|^2\notag\\
    &&\le \frac{d^2+2d(Z_D-1)^2}{(Z_D-1)^4}\sum_{i=1}^Np_i^2E^2\eta_{t}^2G^2,
\end{eqnarray}
 and 
\begin{equation}
\label{BoundOfJ_IFULL}
    \begin{split}
        J^i_{t+1} &= \mathbb{E}\left\|\widetilde{\mathbf{U}}^i_{t+1} - \mathbf{U}^i_{t+1}\right\|^2
        \leq \frac{d}{2(Z_U-1)^2}\mathbb{E}\|\mathbf{U}^i_{t+1}\|^2\\
        &\le \frac{2dE^2\eta_{t}^2G^2}{(Z_U-1)^2}.
    \end{split}
\end{equation}

According to \eqref{BoundOfJ_DFULL} and \eqref{BoundOfJ_IFULL}, we can get
\begin{eqnarray}
\label{VarianceOfWAndW*FULL}
    &&\mathbb{E}\left\|\bar{\mathbf{w}}_{t+1} - \mathbf{w}^*\right\|^2\notag\\
    &&\leq (1-\eta_t\mu)\mathbb{E}\left\|\bar{\mathbf{w}}_{t} - \mathbf{w}^*\right\|^2 + \eta_t^2\tau,
\end{eqnarray}
where $\tau=\sum_{i=1}^Np_i^2\sigma_i^2+6L\Gamma+8(E-1)^2G^2+(\frac{d^2+2d(Z_D-1)^2}{(Z_D-1)^4}+\frac{2d}{(Z_U-1)^2})E^2G^2\sum_{i=1}^Np_i^2$.

We also defined $\Delta_{t+1}=\mathbb{E}\left\|\bar{\mathbf{w}}_{t+1}-\mathbf{w}^*\right\|^2$. According to \eqref{VarianceOfWAndW*FULL}, we can get
$$\Delta_{t+1}\leq(1-\eta_t\mu)\Delta_t + \eta_t^2\tau.$$
Using a similar derivation method as before, we can derive $\Delta_{t}\leq\frac{v}{\gamma+t}$, where $v=\max\{\frac{\beta^2\tau}{\beta\mu-1}, \gamma\Delta_0\}$. 

According to the L-smooth and strong convexity of $F()$, we can get
$$\mathbb{E}[F(\bar{\mathbf{w}}_t)]-F^*\le\frac{L}{2}\Delta_t\le\frac{L}{2}\frac{v}{\gamma+t}.$$
Let $\beta=\frac{2}{\mu},\gamma=\max\{8\frac{L}{\mu}-1,E\},\kappa=\frac{L}{\mu},\eta_t=\frac{2}{\mu}\frac{1}{\gamma+t}$, so
\begin{equation}
    \mathbb{E}[F(\bar{\mathbf{w}}_t)]-F^*\le\frac{2\kappa}{\gamma+t}(\frac{\tau}{\mu}+2L\left\|\bar{\mathbf{w}}_{0}-\mathbf{w}^*\right\|^2).
\end{equation}

\section{Proof of convergence with partial client participation}


Considering the network situation and the symmetry of the model between clients, in each round of training, the server broadcasts the compressed model updates to all clients, but only $K$ clients are randomly selected for training according to the weight probabilities $p_1, p_2, \dots p_N$ with replacement. The selected clients will use local data for training. After $E$ local epochs, the clients will compress the model updates and upload the updates to the server for aggregation.

When $t+1\in\mathcal{I}$, we use $\mathcal{K}_{t+1}=\{i_1, i_2, \dots i_K\}$ to represent the set of clients selected for training. Therefore, the update process of the model can be defined as
\begin{equation}
    \label{EQ:LocalUpdate}
    \begin{split}
        \mathbf{v}^i_{t+1}=\mathbf{w}^i_{t}-\eta_t\nabla F_i(\mathbf{w}^i_t,\mathcal{B}^i_t).
    \end{split}
\end{equation}
\begin{equation}
    \mathbf{w}^i_{t+1}=\left\{
        \begin{aligned}
        &\mathbf{v}^i_{t+1}, \quad if\ t+1\ \notin \mathcal{I},\\
        &\mathbf{w}_{t_0}^i - \widetilde{\mathbf{D}}_{t+1}, \quad if\ t+1\ \in \mathcal{I}, \\
        \end{aligned}
        \right.
\end{equation}
where $\mathbf{D}_{t+1}=\frac{1}{K}\sum_{i\in\mathcal{K}_{t+1}}\widetilde{\mathbf{U}}^i_{t+1}$ and $\mathbf{U}^i_{t+1}=\sum_{j=t_0}^{t} \eta_{j} \nabla F_i(\mathbf{w}^i_{j},\mathcal{B}^i_{j})$.

\subsection{Key Lemmas}
\begin{lemma}
	\label{Lemma:Unbias}
	(Unbiased sampling scheme and  compression)
	Let Assumptions \ref{Assump:ConSmoo}-\ref{Assump:SelectScheme} hold and $t+1 \in \mathcal{I}$, we can get
	\begin{equation}
		\label{EQ:UnbiasedHK}
		\mathbb{E}[\bar{\mathbf{w}}_{(t+1)}]=\bar{\mathbf{v}}_{(t+1)}.
	\end{equation}
\end{lemma}
\begin{proof}
	In our analysis, there are two sources of randomness. One is random selection by the client and the other is caused by MUCSC compression. Because both the selection of the clients and outputs of the MUCSC compression algorithm are unbiased,  \eqref{EQ:UnbiasedHK} holds.
\end{proof}

To derive the convergence rate, we need to bound the variance of $\bar{\mathbf{w}}_t$. According to Eq. \eqref{EQ:GlobalUpdateOfCompAndPartial}, we can derive: 
\begin{lemma}
	\label{Lemma:PartialClientsComp}(Bounding the variance of $\bar{\mathbf{w}}_t$ cause by compression and client selection) Let Assumptions \ref{Assump:ConSmoo}-\ref{Assump:BoundG} hold.  When $t+1 \in \mathcal{I}$, $\eta_t$ is non-increasing and $\eta_t\le2\eta_{t+E}$ for all $t$, then we have 
	\begin{equation}
		\label{EQ:Lemma7}
		\mathbb{E}\left\|\bar{\mathbf{v}}_{t+1}-\bar{\mathbf{w}}_{t+1}\right\|^2\le J^{D}_{t+1} + \frac{1}{K}\left(\sum_{i=1}^Np_{i} J^{i}_{t+1} +  4\eta_{t}^2E^2G^2\right),
	\end{equation}
	where $J^D_{t+1}=\mathbb{E}\|\widetilde{\mathbf{D}}_{t+1} - \mathbf{D}_{t+1}\|^2$ and $J^i_{t+1}=\mathbb{E}\|\widetilde{\mathbf{U}}^i_{t+1} - \mathbf{U}^i_{t+1}\|^2$.
\end{lemma}

Please refer to Appendix \ref{ProofOfLemma7} for the detailed proof.

\subsection{Variance between $\bar{\mathbf{w}}_{t+1}$ and $\bar{\mathbf{v}}_{t+1}$} \label{ProofOfLemma7}
When $t+1\in\mathcal{I}$, we calculate the upper bound of the variance of $\bar{\mathbf{w}}_{t+1}$.
\begin{equation}
\label{VarianceOfWAndVPartial}
    \begin{split}
        &\mathbb{E}\left\|\bar{\mathbf{w}}_{t+1}-\bar{\mathbf{v}}_{t+1}\right\|^2\\
        &=\mathbb{E}\left\|\widetilde{\mathbf{D}}_{t+1}-\sum_{i=1}^Np_i\mathbf{U}^i_{t+1}\right\|^2\\
        &=\mathbb{E}\left\|\widetilde{\mathbf{D}}_{t+1}- \mathbf{D}_{t+1} + \frac{1}{K}\sum_{i\in\mathcal{K}_{t+1}}\widetilde{\mathbf{U}}^i_{t+1} -\sum_{i=1}^Np_i\mathbf{U}^i_{t+1}\right\|^2\\
        &=\mathbb{E}\left\|\widetilde{\mathbf{D}}_{t+1} - \mathbf{D}_{t+1}\right\|^2 + \mathbb{E}\left\|\frac{1}{K}\sum_{i\in\mathcal{K}_{t+1}}\widetilde{\mathbf{U}}^i_{t+1} -\sum_{i=1}^Np_i\mathbf{U}^i_{t+1}\right\|^2.
    \end{split}
\end{equation}

We use $J^D_{t+1}$ to represent the variance produced by the compression of model updates aggregated on the server side, that is $J^D_{t+1}=\mathbb{E}\left\|\widetilde{\mathbf{D}}_{t+1} - \mathbf{D}_{t+1}\right\|^2$.

Next we analyze the second term in \eqref{VarianceOfWAndVPartial}.
\begin{eqnarray}
    &&\mathbb{E}\left[\left\|\frac{1}{K}\sum_{i\in\mathcal{K}_{t+1}} \widetilde{\mathbf{U}}^i_{t+1} - \sum_{i=1}^Np_i\mathbf{U}^i_{t+1}\right\|^2\right]\notag\\
    &&=\frac{1}{K^2}\mathbb{E}\left[\left\|\sum_{l=1}^K\left( \widetilde{\mathbf{U}}^{i_l}_{t+1} - \sum_{i=1}^Np_i\mathbf{U}^i_{t+1}\right)\right\|^2\right]\notag\\
    &&=\frac{1}{K^2}\sum_{l=1}^K\mathbb{E}\left[\left\| \widetilde{\mathbf{U}}^{i_l}_{t+1} - \sum_{i=1}^Np_i\mathbf{U}^i_{t+1}\right\|^2\right]\notag\\
    &&=\frac{1}{K^2}\sum_{l=1}^K\mathbb{E}\left[
    \sum_{i^{'}=1}^Np_{i^{'}}
    \left\| \widetilde{\mathbf{U}}^{i^{'}}_{t+1} - \sum_{i=1}^Np_i\mathbf{U}^i_{t+1}\right\|^2
    \right]\notag\\
    &&=\frac{1}{K}\sum_{i^{'}=1}^Np_{i^{'}}\mathbb{E}\left\| \widetilde{\mathbf{U}}^{i^{'}}_{t+1} - \sum_{i=1}^Np_i\mathbf{U}^i_{t+1}\right\|^2
    \label{VarianceCauseByCompressionAndSample}
\end{eqnarray}

We can make
\begin{eqnarray}
    \label{EQ:ECOMP}
        &&\mathbb{E}\left\| \widetilde{\mathbf{U}}^{i^{'}}_{t+1} - \sum_{i=1}^Np_i\mathbf{U}^i_{t+1}\right\|^2\\
        &&=\mathbb{E}\left\|\widetilde{\mathbf{U}}^{i^{'}}_{t+1}-\mathbf{U}^{i^{'}}_{t+1}\right\|^2 +  \mathbb{E}\left\|\sum_{i=1}^Np_i\mathbf{U}^i_{t+1}-\mathbf{U}^{i^{'}}_{t+1}\right\|^2.\notag
\end{eqnarray}
The first term in the above formula is the variance caused by the compression of model updates on the client side, and we use $J^{i}_{t+1}=\mathbb{E}\left\|\widetilde{\mathbf{U}}^{i}_{t+1}-\mathbf{U}^{i}_{t+1}\right\|^2$ to represent it. Next we analyze the second item.
\begin{eqnarray}
    &&\mathbb{E}\left[\sum_{i^{'}=1}^Np_{i^{'}}\left\|\mathbf{U}^{i^{'}}_{t+1}-\sum_{i=1}^Np_i\mathbf{U}^i_{t+1}\right\|^2\right]\notag\\
    &&\le \mathbb{E}\left[\sum_{i^{'}=1}^Np_{i^{'}}\left\|\mathbf{U}^{i^{'}}_{t+1}\right\|^2\right]\\
    &&\le \sum_{i^{'}=1}^Np_{i^{'}}\mathbb{E}\left[\sum_{j=t_0}^{t}E\eta_j^2\left\|\nabla F_{i^{'}}(\mathbf{w}^{i^{'}}_j,\mathcal{B}^{i^{'}}_j)
    \right\|^2\right]\le 4\eta_t^2E^2G^2,\notag
\end{eqnarray}
where the first inequality results from $\sum_{i^{'}=1}^Np_{i^{'}}\mathbf{U}^{i^{'}}_{t+1}=\sum_{i=1}^Np_i\mathbf{U}^i_{t+1}$ and $\mathbb{E}\|\mathbf{x}-\mathbb{E}\mathbf{x}\|^2\le\mathbb{E}\|\mathbf{x}\|^2$.

Therefore we can get
\begin{eqnarray}
    &&\mathbb{E}\left\|\bar{\mathbf{w}}_{t+1}-\bar{\mathbf{v}}_{t+1}\right\|^2\notag\\
    &&=\mathbb{E}\left\|\widetilde{\mathbf{D}}_{t+1} - \mathbf{D}_{t+1}\right\|^2 +\notag\\
    &&\qquad\mathbb{E}\left\|\frac{1}{K}\sum_{i\in\mathcal{K}_{t+1}}\widetilde{\mathbf{U}}^i_{t+1} -\sum_{i=1}^Np_i\mathbf{U}^i_{t+1}\right\|^2\notag\\
    &&=J^D_{t+1} + \frac{1}{K}\sum_{i^{'}=1}^Np_{i^{'}}\notag\\
    &&\qquad\left[\mathbb{E}\left\|\widetilde{\mathbf{U}}^{i^{'}}_{t+1}-\mathbf{U}^{i^{'}}_{t+1}\right\|^2 +\mathbb{E}\left\|\sum_{i=1}^Np_i\mathbf{U}^i_{t+1}-\mathbf{U}^{i^{'}}_{t+1}\right\|^2\right]\notag\\
    &&\le J^{D}_{t+1} + \frac{1}{K}\left(\sum_{i=1}^Np_{i} J^{i}_{t+1} +  4\eta_{t}^2E^2G^2\right).
\end{eqnarray}

\subsection{Convergence of the algorithm} \label{ProofOfTheorem3And4}
Next, we analyze the interval between $\bar{\mathbf{w}}_{t+1}$ and $\mathbf{w}^*$.
\begin{equation}
    \label{PartialClientGap}
    \begin{split}
        &\mathbb{E}\left\|\bar{\mathbf{w}}_{t+1} - \mathbf{w}^*\right\|^2\\
        &=\mathbb{E}\left\|\bar{\mathbf{w}}_{t+1} - \bar{\mathbf{v}}_{t+1} + \bar{\mathbf{v}}_{t+1} - \mathbf{w}^*\right\|^2\\
        &\le (1-\eta_t\mu)\left\|\bar{\mathbf{w}}_{t}-\mathbf{w}^*\right\|^2+\eta_t^2\tau+\psi_{t+1},
    \end{split}
\end{equation}
where $\tau=\frac{\sum_{i=1}^Np_i^2\sigma_i^2}{B}+6L\Gamma+8(E-1)^2G^2+\frac{4E^2G^2}{K}$ and $\psi_{t+1}=J^D_{t+1} + \frac{1}{K}\sum_{i=1}^Np_{i} J^{i}_{t+1}$.

We can find that \eqref{AllClientGap} and \eqref{PartialClientGap} are basically similar, so we can directly use similar proof methods. The final conclusion we can draw is
\begin{eqnarray}
    &&\mathbb{E}\left\|\bar{\mathbf{w}}_{t+1}-\mathbf{w}^*\right\|^2\leq \frac{v_t}{\gamma+t},
\end{eqnarray}

where $v_t=\max\{\frac{\beta^2\tau}{\beta\mu-1}+\sum_{i=1}^t[(i+\gamma)\psi_i], \gamma\|\mathbf{w}_0-\mathbf{w}^*\|^2\}$ for a diminishing stepsize $\eta_t=\frac{\beta}{t+\gamma}$ for some $\beta>\frac{1}{\mu}$ and $\gamma>0$.

Considering the L-smooth and strong convexity of $F()$, we can get
\begin{equation}
    \mathbb{E}[F(\bar{\mathbf{w}}_t)]-F(\mathbf{w}^*)\leq \frac{L}{2}\frac{v_t}{t+\gamma}.
\end{equation}

Similarly, in order to make the result of convergence more clear, let's analyze the upper bound of $J^D_{t+1}$ and $J^i_{t+1}$.

\begin{eqnarray}
    &&J^D_{t+1}=\mathbb{E}\left\|\widetilde{\mathbf{D}}_{t+1} - \mathbf{D}_{t+1}\right\|^2\notag\\
    &&\le \frac{d}{2(Z_D-1)^2K^2} \mathbb{E}\left\|\sum_{i\in\mathcal{K}_{t+1}}\widetilde{\mathbf{U}}^i_{t+1}\right\|^2\notag\\
    &&\le \frac{d}{2(Z_D-1)^2K}\sum_{i\in\mathcal{K}_{t+1}} \mathbb{E}\left\|\widetilde{\mathbf{U}}^i_{t+1}\right\|^2\notag\\
    &&\le \frac{d}{2(Z_D-1)^2} \sum_{i=1}^N p_i\mathbb{E}\left\|\widetilde{\mathbf{U}}^i_{t+1}\right\|^2\notag\\
    &&\le \frac{d}{2(Z_D-1)^2}\sum_{i=1}^Np_i\notag\\
    &&\qquad\left(\mathbb{E}\left\|\widetilde{\mathbf{U}}^i_{t+1} - \mathbf{U}^i_{t+1}\right\|^2+\mathbb{E}\left\| \mathbf{U}^i_{t+1}\right\|^2\right)\notag\\
    &&\le \frac{d^2+2d(Z_D-1)^2}{4(Z_D-1)^4}\sum_{i=1}^Np_i\mathbb{E}\left\| \mathbf{U}^i_{t+1}\right\|^2\notag\\
    &&\le \frac{d^2+2d(Z_D-1)^2}{(Z_D-1)^4}E^2\eta_{t}^2G^2.
\end{eqnarray}
Therefore
\begin{eqnarray}
    \mathbb{E}\left\|\bar{\mathbf{w}}_{t+1} - \mathbf{w}^*\right\|^2\leq (1-\eta_t\mu)\mathbb{E}\left\|\bar{\mathbf{w}}_{t} - \mathbf{w}^*\right\|^2 + \eta_t^2\tau,
\end{eqnarray}
where $\tau=\sum_{i=1}^Np_i^2\sigma_i^2+6L\Gamma+8(E-1)^2G^2+\frac{4}{K}E^2G^2+(\frac{d^2+2d(Z_D-1)^2}{(Z_D-1)^4}+\frac{2d}{K(Z_U-1)^2})E^2G^2$.

Using a similar derivation method as before, we can derive $\Delta_{t}\leq\frac{v}{\gamma+t}$, where $v=\max\{\frac{\beta^2\tau}{\beta\mu-1}, \gamma\Delta_0\}$ for $\eta_t=\frac{\beta}{t+\gamma}$, $\beta>\frac{1}{\mu}$ and $\gamma>0$ such that $\eta_1\leq\min\{\frac{1}{\mu}, \frac{1}{4L}\}$ and $\eta_t\leq2\eta_{t+E}$.

Let $\beta=\frac{2}{\mu},\gamma=\max\{8\frac{L}{\mu}-1,E\},\kappa=\frac{L}{\mu},\eta_t=\frac{2}{\mu}\frac{1}{\gamma+t}$, so
\begin{equation}
    \mathbb{E}[F(\bar{\mathbf{w}}_t)]-F^*\le\frac{2\kappa}{\gamma+t}(\frac{\tau}{\mu}+2L\left\|\bar{\mathbf{w}}_{0}-\mathbf{w}^*\right\|^2).
\end{equation}

\ifCLASSOPTIONcaptionsoff
  \newpage
\fi

%

\bibliographystyle{IEEEtran} 
\bibliography{reference}
\end{document}